\documentclass{article}

\usepackage{microtype}
\usepackage{graphicx}
\usepackage{subcaption}
\usepackage{algorithm}
\usepackage{wrapfig}
\usepackage{booktabs} 
\usepackage{multirow}
\usepackage{makecell}
\usepackage{hyperref}
\usepackage{url}
\usepackage{xurl} 
\usepackage{amsmath,amssymb}
\usepackage{enumitem}
\usepackage{xspace}
\usepackage[table]{xcolor}
\usepackage[T1]{fontenc}
\usepackage{pgf}

\usepackage{tcolorbox}
\definecolor{PromptNavy}{HTML}{244968}
\definecolor{PromptOrange}{HTML}{AD5A2D}
\newtcolorbox{promptcard}[2]{
  colback=#1!4!white, colframe=#1!55!white,
  colbacktitle=#1, coltitle=white,
  title={#2}, fonttitle=\small\bfseries\sffamily,
  fontupper=\small\sffamily, fontlower=\small\sffamily,
  boxrule=0.5pt, arc=2pt,
  left=10pt, right=10pt, top=7pt, bottom=7pt,
  toptitle=5pt, bottomtitle=5pt,
  before skip=10pt, after skip=10pt,
  before upper={\setlength{\parindent}{0pt}\setlength{\parskip}{5pt}},
  before lower={\setlength{\parindent}{0pt}\setlength{\parskip}{5pt}}
}
\newcommand{\promptrole}[2]{%
  {\scriptsize\bfseries\sffamily\color{#1}#2}\par\smallskip}

\usepackage[preprint]{icml2026}

\usepackage{mathtools}
\usepackage{amsthm}

\definecolor{OPDBlue}{HTML}{5C83AA}
\definecolor{OPDRed}{HTML}{B96B69}
\definecolor{groupgray}{HTML}{F2F3F5}
\definecolor{deltaGray}{HTML}{6D737B}
\definecolor{methodBg}{HTML}{E4EEF9}

\newcommand{\OPDBlueMax}{32}
\newcommand{\OPDRedMax}{16}
\newcommand{\OPDGamma}{1.5}

\newcommand{\OPDShade}[4]{%
  \pgfmathsetmacro{\OPDDelta}{(#1)-(#2)}%
  \ifdim\OPDDelta pt>0pt\relax
  \pgfmathsetmacro{\OPDIntensity}{%
  \OPDBlueMax*pow(min(1,max(0,\OPDDelta/max(#3,0.0001))),\OPDGamma)}%
  \edef\OPDPaint{\noexpand\cellcolor{OPDBlue!\OPDIntensity!white}}%
  \OPDPaint
  \else\ifdim\OPDDelta pt<0pt\relax
  \pgfmathsetmacro{\OPDIntensity}{%
  \OPDRedMax*pow(min(1,max(0,-\OPDDelta/max(#4,0.0001))),\OPDGamma)}%
  \edef\OPDPaint{\noexpand\cellcolor{OPDRed!\OPDIntensity!white}}%
  \OPDPaint
  \fi\fi
}

\providecommand{\shpn}[4]{}
\renewcommand{\shpn}[4]{\OPDShade{#1}{#2}{#3}{#4}#1}
\providecommand{\shpbn}[4]{}
\renewcommand{\shpbn}[4]{\OPDShade{#1}{#2}{#3}{#4}\textbf{#1}}
\providecommand{\tv}[1]{\textcolor{deltaGray}{#1}}

\usepackage[capitalize,noabbrev]{cleveref}

\theoremstyle{plain}
\newtheorem{theorem}{Theorem}[section]
\newtheorem{proposition}[theorem]{Proposition}

\theoremstyle{definition}

\theoremstyle{remark}

\usepackage[textsize=tiny]{todonotes}

\newcommand{\Stu}{\ensuremath{\pi_S}\xspace}
\newcommand{\Tea}{\ensuremath{\pi_T}\xspace}
\newcommand{\Synth}{\ensuremath{S_\theta}\xspace}
\newcommand{\method}{\texttt{A2D}\xspace}
\newcommand{\huggingface}{\raisebox{-1.5pt}{\includegraphics[height=1.05em]{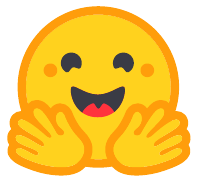}}\xspace}
\newcommand{\github}{\raisebox{-1.5pt}{\includegraphics[height=1.05em]{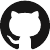}}\xspace}

\newcommand{\Topk}{\ensuremath{\mathrm{Top}\text{-}K}}

\icmltitlerunning{Ask to Distill: Learning to Synthesize Data for On-Policy Distillation}

\begin{document}

\twocolumn[
  \icmltitle{Ask to Distill: Learning to Synthesize Data for On-Policy Distillation}



  \icmlsetsymbol{equal}{*}

  \begin{icmlauthorlist}
    \icmlauthor{Jiashuo Sun}{uiuc}
    \icmlauthor{Hao Zhang}{hkust}
    \icmlauthor{Xueqiang Xu}{uiuc}
    \icmlauthor{Heng Wang}{uiuc}
    \icmlauthor{Yixuan Xie}{uiuc}
    \icmlauthor{Pengcheng Jiang}{uiuc}
    \icmlauthor{Zhenrui Yue}{uiuc}
    \icmlauthor{Jiajun Fan}{uiuc}
    \icmlauthor{Peiran Li}{tamu}
    \icmlauthor{Ge Liu}{uiuc}
    \icmlauthor{Shaowen Wang}{uiuc}
    \icmlauthor{Jiawei Han}{uiuc}
  \end{icmlauthorlist}

  \begin{center}
    \small
    \github\ \href{https://github.com/GasolSun36/Ask_to_Distill}{\texttt{GitHub}}
    \hspace{3em}
    \huggingface\ \href{https://huggingface.co/collections/gasolsun/ask-to-distill-model-and-data-collection}{\texttt{Model and Data Collections}}
    \hspace{3em}
    \huggingface\ \href{https://huggingface.co/datasets/gasolsun/Ask_to_Distill-eval}{\texttt{Eval Data}}
  \end{center}


  \icmlaffiliation{uiuc}{University of Illinois Urbana-Champaign, Urbana, IL, USA}
  \icmlaffiliation{hkust}{The Hong Kong University of Science and Technology, Hong Kong SAR, China}
  \icmlaffiliation{tamu}{Texas A\&M University, College Station, TX, USA}

  \icmlcorrespondingauthor{Jiashuo Sun}{jiashuo5@illinois.edu}

  \icmlkeywords{On-policy distillation, synthetic data, reinforcement learning}

  \vskip 0.3in
]



\printAffiliationsAndNotice{}  

\begin{abstract}
On-policy distillation (OPD) can succeed with surprisingly few training questions, motivating a view of questions as carriers of distributional knowledge transfer. How can we assess whether a newly generated question is valuable for distillation, and learn to generate better questions---without access to its correct answer? We introduce Ask-to-Distill (\method), a framework that treats synthetic questions as distribution elicitors and optimizes the teacher supervision they expose along student-generated trajectories. A student-initialized synthesizer is optimized using a distributional surrogate for question value, computed from a frozen student--teacher pair without reference answers, correctness rewards, or per-candidate student updates. The reward combines teacher--student disagreement with teacher probability mass on the student's candidate support, together with length and diversity controls. The generated questions are filtered and deduplicated into a prompt-only bank for standard OPD. Across mathematical reasoning, code generation, and agentic tasks, our method improves average performance and teacher-gap recovery over unoptimized student-generated prompts and KL-filtered real questions, while matching or exceeding real-bank OPD in the evaluated settings. Signal analyses further show that optimized questions can elicit denser supervision and more total signal from shorter student responses, supporting elicited distributional supervision as an objective for synthetic distillation data. Our code, data and models are publicly available at \url{https://github.com/GasolSun36/Ask_to_Distill}.
\end{abstract}

\begin{figure}[!t]
  \centering
  \includegraphics[width=0.49\textwidth]{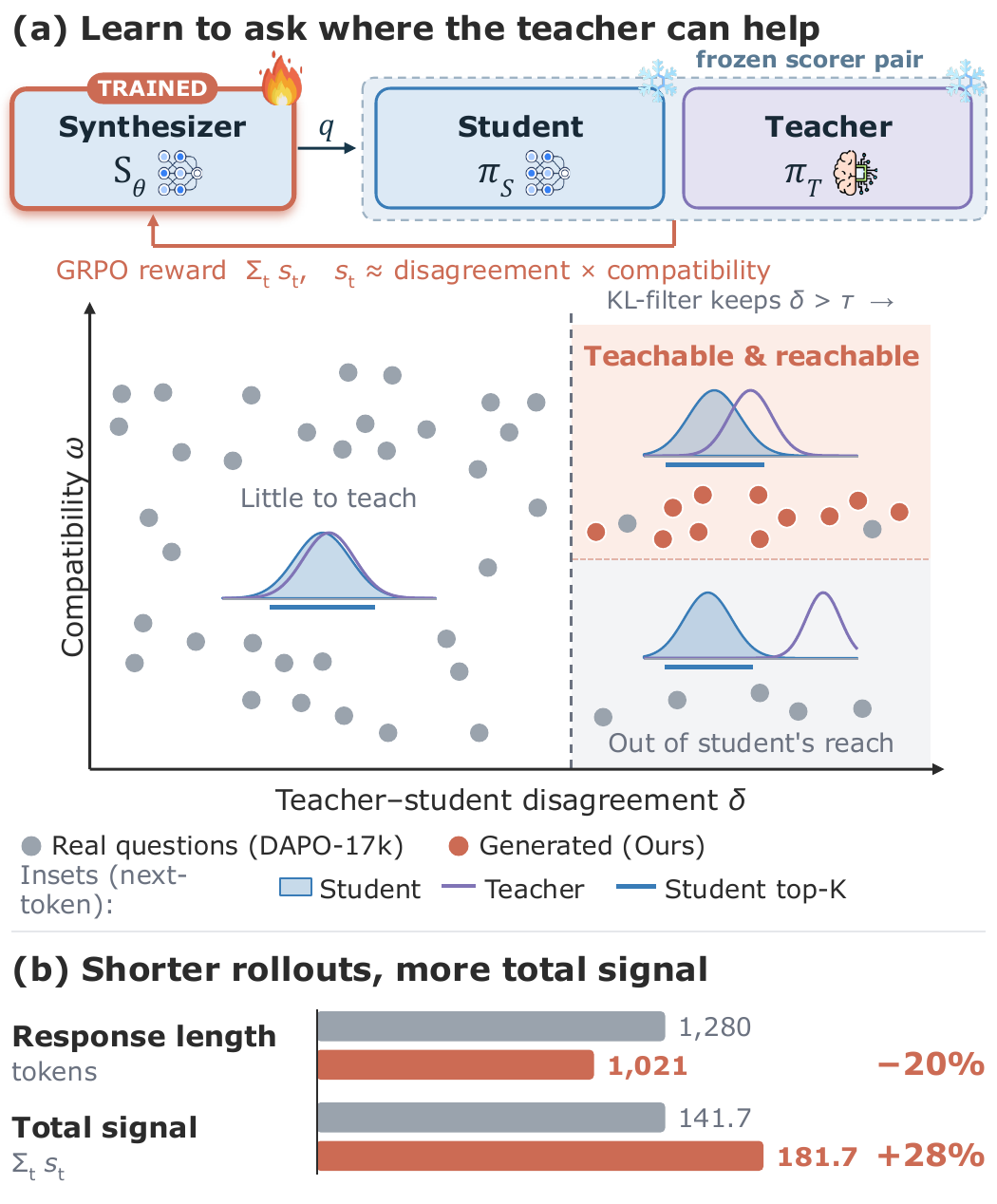}
  \caption{\textbf{\method learns to ask questions where the teacher can help.}
(a) Schematic. A student-initialized synthesizer is RL-trained against a frozen
student--teacher pair to elicit disagreement $\delta$ that remains within the student's
top-$K$ support (compatibility $\omega$). Real questions mostly elicit little
disagreement or disagreement outside the student's reach, which KL filtering cannot
separate. Insets: typical student and teacher next-token distributions per region.
(b) Generated questions yield 20\% shorter responses and 28\% more total signal than
DAPO-17k-Math.}
\end{figure}

\section{Introduction}
\label{sec:introduction}

On-policy distillation (OPD) transfers the behavior of a strong teacher to a student on trajectories sampled by the student itself~\citep{gkd,minillm}. By supervising the states the student actually visits, OPD mitigates the distribution mismatch associated with learning only from fixed teacher demonstrations and provides dense token-level guidance without requiring a task-specific correctness verifier. Although the teacher supplies this supervision, the training questions determine which trajectories the student generates and where teacher feedback becomes available. Constructing effective distillation data therefore requires understanding not only what a question asks, but also what supervision it elicits from a particular student--teacher pair.

\begin{figure*}[!t]
  \centering
  \includegraphics[width=\textwidth]{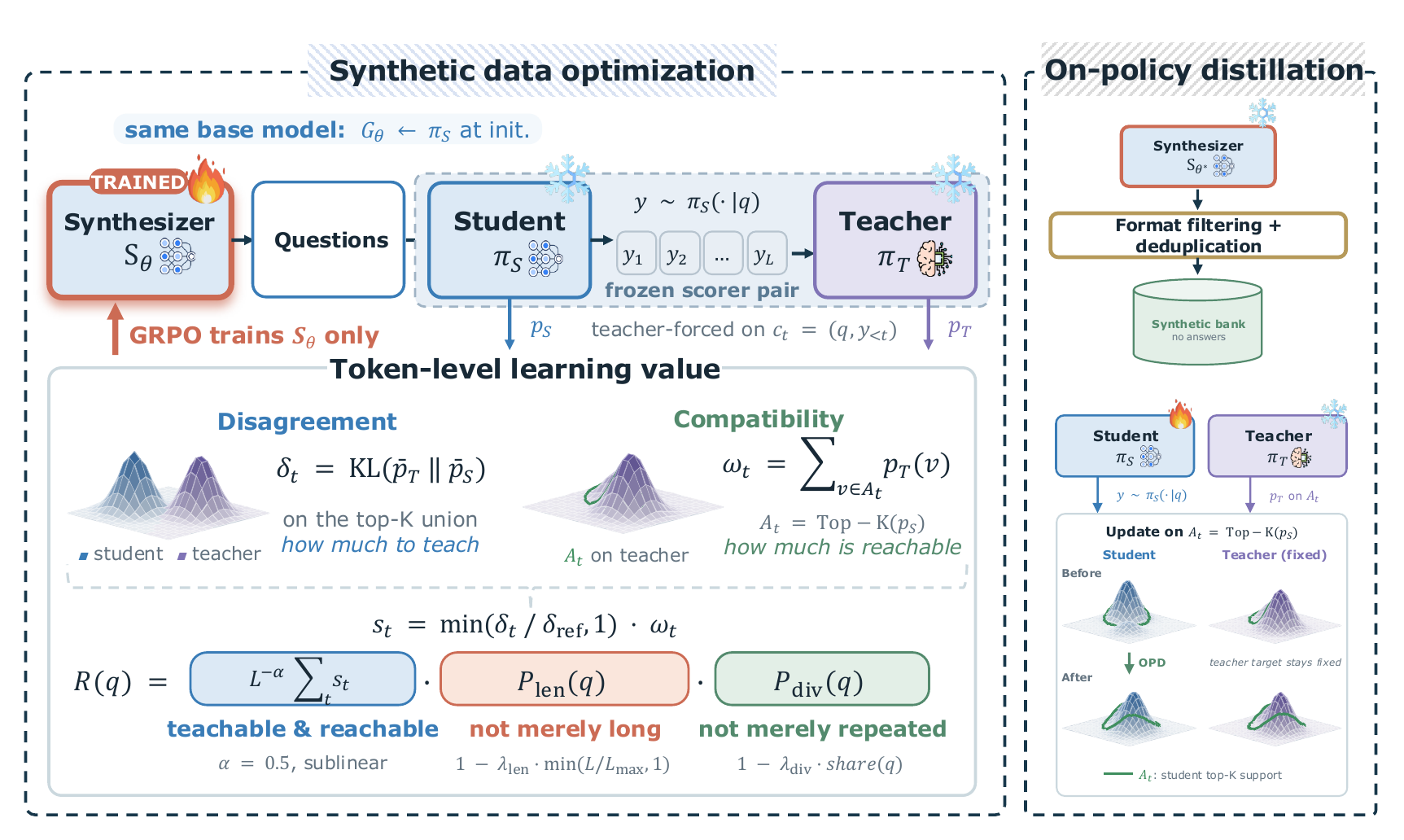}
  \caption{Overview of \method. A student-initialized synthesizer is optimized using learning-value signals from a frozen student--teacher pair. The generated questions are filtered and deduplicated into a prompt-only synthetic bank, which is then used for on-policy distillation with a frozen teacher providing supervision on the student's top-$K$ candidates.}
  \label{fig:ask_to_distill}
\end{figure*}

Recent studies suggest that OPD can be highly prompt-efficient: small subsets can approach full-data performance, and even a single question can recover much of the gain obtained from a large question bank~\citep{rethink_opd2,opd-data-efficiency}. Moreover, self-generated questions can support effective OPD without external seed examples~\citep{data-free}. These findings motivate a different perspective on OPD training data, centered on the supervision that questions elicit. In supervised fine-tuning, synthetic examples provide target responses for imitation~\citep{self-instruct,xu2024wizardlm}; in self-play reinforcement learning, generated tasks provide opportunities to learn from answer-quality or task-success feedback~\citep{huang2025rzero,chen2025selfquestioning}. OPD instead obtains supervision from the teacher's token distributions on student-generated prefixes. A question therefore serves as a \emph{distribution elicitor}: it guides the student into states where teacher--student differences become available for learning. From this perspective, training questions act as carriers of distributional knowledge transfer, rather than directly supplying the target knowledge through reference answers. Their distillation value depends on the usable teacher supervision exposed along the induced trajectories for a particular student--teacher pair, consistent with findings that prompt utility is pair-dependent~\citep{wang2026promptcountdatashapes}. Although OPD supplies teacher supervision without reference answers, optimizing a question synthesizer requires an additional signal: feedback on which questions are worth generating. This raises a central question: \emph{without access to the correct answer to a newly generated question, how can we assess its value for a target student--teacher pair and use that assessment to learn to generate better questions?} We address this question by scoring the distributional supervision exposed along student-generated trajectories, using a frozen student--teacher pair without performing a student update for each candidate.

Turning elicited supervision into an optimization objective requires distinguishing disagreement from useful correction. Token-selection and reweighting methods already exploit the unequal value of supervision across response positions~\citep{tip_opd,filter-opd}. In particular, support-aware analyses of token teachability show that large divergence can reflect either corrective teacher preferences within the student's candidate support or a mismatch concentrated outside that support~\citep{opd_token_teach}. Building on this characterization, we combine teacher--student discrepancy with teacher probability mass on the student's candidate tokens as a surrogate for local learning value. We adapt this token-level principle to question generation by aggregating the signal along student trajectories. This introduces additional challenges: a synthesizer may increase its reward through longer responses or repeatedly produce a narrow family of questions. The question-level objective must therefore balance elicited signal with trajectory length and question diversity.

We introduce Ask-to-Distill (\method), a self-contained framework that optimizes synthetic questions for a target student--teacher pair before performing OPD. We initialize the question-generation policy from the student and train it with reinforcement learning, while keeping the scoring student and teacher frozen. Its reward aggregates support-aware token signals using sublinear length normalization, a soft response-budget penalty, and a batch-diversity term. We then curate the generated questions through format filtering and lexical and semantic deduplication, and perform standard OPD on the resulting prompt-only dataset. Synthesizer training reuses the target student--teacher pair and requires neither external training questions nor auxiliary models. After training, a frozen Qwen3-Embedding-4B model is used only for semantic deduplication of the generated question bank; it does not participate in synthesizer optimization. Our method operates on the prompt-construction axis, which is complementary to approaches that select supervision positions, reweight trajectories, or modify the downstream distillation procedure.

Our contributions are threefold: (i) we formulate synthetic prompt construction as optimizing the distributional supervision elicited from a target student--teacher pair; (ii) we develop a question-level reward that turns support-aware teacher--student signals into feedback for synthesizer optimization, without reference answers, task-specific correctness rewards, or per-candidate student updates, while controlling response length and question diversity; and (iii) we validate the framework across multiple model scales and mathematical reasoning, code generation, and agentic tasks through downstream comparisons, reward-component ablations, and learning-signal analyses. \method exceeds real-bank OPD on the average scores of the evaluated student--teacher pairs. Our analyses further show that optimized questions can elicit denser distributional signals from shorter student responses, supporting the value of optimizing elicited supervision beyond simply increasing rollout length.

\section{Method}
\label{sec:method}

\subsection{Problem formulation: synthetic prompts as distribution elicitors}
\label{sec:prelim}

On-policy distillation (OPD) provides teacher supervision on trajectories sampled by the student. Unlike synthetic examples used for supervised fine-tuning, a question used for OPD need not include a reference response: it supplies the context from which the student generates a trajectory and the teacher provides token-level supervision. Our synthetic data therefore consist of questions only, without ground-truth answers or reference solutions; neither synthesizer reward computation nor downstream OPD requires such annotations. We therefore study synthetic questions as \emph{distribution elicitors}, whose value depends on the supervision they expose for a particular student--teacher pair.

Let \Stu and \Tea denote the student and teacher, with policies $\pi_S$ and $\pi_T$. Given a question $q$, the student samples a response $y \sim \pi_S(\cdot \mid q)$ of length $L$. At each response position $t$, both models are evaluated on the student-generated prefix $c_t = (q, y_{<t})$, producing next-token distributions $\pi_S(\cdot \mid c_t)$ and $\pi_T(\cdot \mid c_t)$. The question thus induces a distribution over the states at which distillation takes place.

Our goal is to construct questions that expose useful teacher supervision along these trajectories. Disagreement alone is insufficient: a question may elicit substantially different predictions while placing little teacher probability mass on the candidates considered by the student's update. We therefore assess questions using both teacher--student discrepancy and compatibility, while retaining basic format and self-containedness requirements during curation.

We formalize the elicited signal through a trajectory-based surrogate:
\begin{equation}
  \mathcal{V}(q;\Stu,\Tea)
  =
  \mathbb{E}_{y\sim\pi_S(\cdot\mid q)}
  \left[
    \frac{1}{L^{\alpha}}
    \sum_{t=1}^{L}s_t
  \right],
  \label{eq:elicitation_value}
\end{equation}
where $s_t$ is a support-aware token score defined in \S\ref{sec:reward}, and $\alpha$ controls the balance between signal density and trajectory length. This quantity measures the distributional signal exposed by a question; it is a surrogate for learning value rather than a direct estimate of downstream performance improvement.

We parameterize the question distribution with a synthesizer $\Synth$, initialized as an independent copy of the initial student. During synthesizer optimization, the scoring student and teacher remain frozen, and only $\theta$ is updated. The synthesizer's reward combines a sampled estimate of Eq.~\eqref{eq:elicitation_value} with response-length and question-diversity controls. The resulting questions are curated and used for downstream OPD, which starts from the original student parameters. Question generation and synthesizer training use only the target student--teacher pair and the student-initialized synthesizer, without external training questions, labeled examples, or auxiliary models. A separate frozen embedding model is used only for semantic deduplication after question generation (\S\ref{sec:bank}).

\subsection{Overview}
\label{sec:overview}

The basic pipeline consists of three stages, each performed once. We also consider repeating the pipeline with the updated student in \S\ref{sec:iterative_training}.

\begin{enumerate}[leftmargin=1.6em,itemsep=2pt,topsep=2pt]
  \item \textbf{Synthetic question optimization} (\S\ref{sec:reward}--\S\ref{sec:synthetic_optimization}). We train $S_\theta$ with group relative policy optimization (GRPO) \cite{grpo}. Each generated question is answered by the frozen student, and the frozen teacher is evaluated on the same student-generated prefixes. Their predictions determine a question-level learning-value reward.
  \item \textbf{Synthetic question curation} (\S\ref{sec:bank}). We sample the trained synthesizer at scale, filter malformed outputs, and remove lexical and semantic duplicates to construct a prompt-only question bank.
  \item \textbf{On-policy distillation} (\S\ref{sec:opd}). We update the original student using teacher supervision on responses sampled from the curated questions, with the same response budget used for question scoring.
\end{enumerate}

The synthesizer and scoring student share their initial parameters but are separate model copies. Only the synthesizer is updated in stage 1; only the student is updated in stage 3. The teacher remains frozen throughout. Algorithm~\ref{alg:ask_to_distill} in the appendix summarizes the full pipeline; hyperparameter values for all stages are reported in Appendix~\ref{sec:implementation}.

\subsection{The learning-value reward}
\label{sec:reward}

\paragraph{Per-position quantities.}
Given a question $q$, we sample a student response $y \sim \pi_S(\cdot\mid q)$ under a response budget $L_{\max}$. Let $L\leq L_{\max}$ be its length and $c_t=(q,y_{<t})$ its prefix at position $t$. We write $p_S(v)=\pi_S(v\mid c_t)$ and $p_T(v)=\pi_T(v\mid c_t)$, suppressing the dependence on $t$ for readability. For a candidate size $K$, we define the student and teacher candidate sets and their union:
\begin{equation}
  \begin{aligned}
    A_t &= \Topk(p_S), &
    B_t &= \Topk(p_T), \\
    U_t &= A_t \cup B_t, &
    |U_t| &\leq 2K.
  \end{aligned}
\end{equation}
For $M\in\{S,T\}$, let
\begin{equation}
  \bar p_M^{U}(v)
  =
  \frac{p_M(v)}
       {\sum_{u\in U_t}p_M(u)},
  \qquad v\in U_t,
\end{equation}
denote the corresponding distribution restricted to and renormalized over $U_t$.

Following support-aware characterizations of token teachability~\citep{opd_token_teach}, we use two quantities to characterize the supervision at each position:
\begin{align}
  \delta_t
  &=
  \mathrm{KL}\!\left(
    \bar p_T^{U}\,\|\,\bar p_S^{U}
  \right)
  =
  \sum_{v\in U_t}
  \bar p_T^{U}(v)
  \log
  \frac{\bar p_T^{U}(v)}
       {\bar p_S^{U}(v)},
  \label{eq:div}\\[2pt]
  \omega_t
  &=
  \sum_{v\in A_t}p_T(v).
  \label{eq:comp}
\end{align}

The \emph{elicited discrepancy} $\delta_t$ measures differences between teacher and student preferences over their combined candidate set. Its teacher-to-student KL direction emphasizes teacher-preferred candidates to which the student assigns relatively low probability. The \emph{compatibility} $\omega_t$ measures the teacher's probability mass on the student's top-$K$ candidates, which are the candidates used by our downstream OPD update. Unlike $\delta_t$, $\omega_t$ uses the original full-vocabulary teacher probabilities without renormalization.

\paragraph{Gating discrepancy by compatibility.}
Large discrepancy can occur even when the teacher places little probability mass on the student's candidate set. Candidates outside this set do not receive explicit candidate-level terms in our downstream update (\S\ref{sec:opd}). We therefore combine discrepancy and compatibility at each position:
\begin{equation}
  s_t
  =
  \min\!\left(
    \frac{\delta_t}
         {\delta_{\mathrm{ref}}+\epsilon},
    1
  \right)
  \omega_t,
  \label{eq:st}
\end{equation}
where $\epsilon$ is a small constant for numerical stability and $\delta_{\mathrm{ref}}$ is a fixed calibration scale. The resulting score lies in $[0,1]$ and favors positions with both substantial discrepancy and teacher probability mass within the student's candidate support. Multiplying the quantities before aggregation requires them to occur at the same position; multiplying their trajectory-level means could instead combine high discrepancy at one position with high compatibility at another. This construction remains a proxy for useful supervision and does not guarantee that every high-scoring position yields a large learning gain.

\paragraph{Calibrating $\delta_{\mathrm{ref}}$.}
We set $\delta_{\mathrm{ref}}$ to the $\kappa$-th percentile of token-level discrepancies measured on $N_{\mathrm{cal}}$ questions sampled from the initial student policy. The frozen student answers these questions under the response budget used for downstream OPD, and the teacher is evaluated on the resulting student-generated prefixes. We pool the resulting $\delta_t$ values to compute the percentile and hold it fixed during synthesizer optimization. This calibration uses no external questions and provides a common discrepancy scale across optimization steps.

\paragraph{Length control.}
We aggregate the token scores using
\begin{equation}
  R_{\mathrm{raw}}(q)
  =
  \frac{1}{L^\alpha}
  \sum_{t=1}^{L}s_t.
  \label{eq:agg}
\end{equation}
Writing $\bar{s}(q)=L^{-1}\sum_{t=1}^{L}s_t$ gives $R_{\mathrm{raw}}(q)=L^{1-\alpha}\bar{s}(q)$. Thus, $\alpha$ controls the balance between average signal and accumulated signal. A per-token mean ($\alpha=1$) gives the same score to trajectories with equal average signal regardless of length, while a sum ($\alpha=0$) grows linearly with length at fixed average signal. We use an intermediate exponent $0<\alpha<1$ to reward sustained signal while reducing the advantage conferred by additional tokens.

This aggregation retains an incentive for longer responses. We moderate it with a multiplicative penalty based on response-budget consumption:
\begin{equation}
  P_{\mathrm{len}}(q)
  =
  1-\lambda_{\mathrm{len}}
  \min\!\left(
    \frac{L}{L_{\max}},1
  \right).
  \label{eq:length_penalty}
\end{equation}
Here, $\lambda_{\mathrm{len}}\in(0,1)$ sets the penalty strength. The penalty decreases with response length and equals $1-\lambda_{\mathrm{len}}$ at the generation cap. Its multiplicative form makes the relative discount independent of the absolute scale of $R_{\mathrm{raw}}$. Together, these controls attenuate length-driven reward growth while retaining a preference for signal accumulated across response positions; they do not eliminate the incentive for longer trajectories.

\paragraph{Diversity control.}
Optimizing the trajectory-level signal alone can concentrate generation on repetitive questions. Following the cluster-based repetition penalty in R-Zero~\citep{huang2025rzero}, we measure within-batch redundancy using BLEU-based text similarity, without extracting hidden-state representations or using an embedding model.

Let $\mathcal{B}=\{q_1,\ldots,q_{N_{\mathcal{B}}}\}$ denote the full batch of questions scored in one synthesizer-training step. We construct a pairwise distance matrix from one minus the sentence-level BLEU similarity between question texts, following the R-Zero implementation, and apply average-linkage agglomerative clustering with distance threshold $\tau$. If $\ell_i$ is the cluster label of $q_i$, its cluster share is
\begin{equation}
  \mathrm{share}(q_i;\mathcal{B})
  =
  \frac{
    \left|\{j:\ell_j=\ell_i\}\right|
  }{N_{\mathcal{B}}}.
\end{equation}
We then define
\begin{equation}
  P_{\mathrm{div}}(q_i;\mathcal{B})
  =
  \max\!\left(
    0,\,
    1-\lambda_{\mathrm{div}}
    \mathrm{share}(q_i;\mathcal{B})
  \right),
\end{equation}
where $\lambda_{\mathrm{div}}>0$ sets the penalty strength.

Questions in larger text-similarity clusters receive a stronger discount. A singleton cluster receives a factor of $1-\lambda_{\mathrm{div}}/N_{\mathcal{B}}$, while for $\lambda_{\mathrm{div}}\geq 1$ a batch assigned entirely to one cluster receives zero reward. The multiplicative form discounts the elicited signal proportionally, avoiding an additive penalty whose relative strength would depend on the signal's numerical scale. Clustering is performed across the full scoring batch, separately from the within-group advantage normalization used by GRPO.

\paragraph{Final question-level reward.}
The reward for a question combines its sampled token signals with the length and diversity factors:
\begin{equation}
  R(q;\mathcal{B})
  =
  \frac{
    P_{\mathrm{len}}(q)
    P_{\mathrm{div}}(q;\mathcal{B})
  }{L^\alpha}
  \sum_{t=1}^{L}s_t.
  \label{eq:reward}
\end{equation}
We make the batch dependence explicit because the diversity factor depends on the other questions scored in the same step. The token-signal term is estimated from one student response per question, providing a single-sample estimate of the expectation in Eq.~\eqref{eq:elicitation_value} before applying the additional controls. Appendix~\ref{app:theory} establishes a compatible-disagreement coverage bound for the normalized reward and relates compatibility to the OPD candidate advantages.

\begin{table*}[!t]
  \centering
  \caption{
    Results across six mathematical reasoning benchmarks for three student--teacher pairs.
    Within each block and metric, shading is relative to Student-Generated OPD
    (white): \textcolor{OPDBlue}{blue} indicates improvement and
    \textcolor{OPDRed}{red} indicates degradation. Positive and negative changes
    are normalised separately across OPD methods and mapped nonlinearly
    (exponent 1.5) to muted palettes.
    Bold indicates the best result among OPD methods, including ties.
    The baselines compare real-bank OPD, untrained student-generated prompts,
    KL-filtered selection, and two additional synthetic-data methods.
  }
  \label{tab:main}

  \footnotesize
  \setlength{\tabcolsep}{6pt}
  \renewcommand{\arraystretch}{1.2}

  \resizebox{\textwidth}{!}{%
    \begin{tabular}{@{}l cccccc @{\hspace{14pt}} c c@{}}
      \toprule
      Method & AIME24 & AIME25 & AIME26 & AMC23 & BRUMO25 & MATH500 & Avg. & Recovery (\%) \\
      \midrule
      \rowcolor{groupgray}
      \multicolumn{9}{@{}l@{}}{\hspace{3pt}\textbf{Qwen3-1.7B (Non-thinking)}\quad{\color{deltaGray}Teacher: Qwen3-4B-Non-Thinking-RL-Math}} \\
      Student
      & 10.0 & 8.3 & 8.3
      & 37.5 & 17.9 & 72.0
      & 25.7 & 0.0 \\
      \tv{Teacher}$^{\dagger}$
      & \tv{41.2} & \tv{33.1} & \tv{34.4}
      & \tv{57.5} & \tv{45.8} & \tv{78.3}
      & \tv{48.4} & \tv{100.0} \\
      \midrule
      \makecell[l]{Student-Generated\\OPD}
      & 27.5 & 22.7 & 21.3
      & 61.6 & 35.0 & 78.4
      & 41.1 & 67.9 \\
      KL-Filtered OPD
      & \shpn{33.5}{27.5}{7.3}{1} & \shpbn{26.7}{22.7}{4.0}{1} & \shpn{23.3}{21.3}{5.2}{1.3}
      & \shpn{64.4}{61.6}{10.3}{1} & \shpn{36.7}{35.0}{1.7}{3.3} & \shpn{80.2}{78.4}{4.0}{1}
      & \shpn{44.1}{41.1}{5.1}{1} & \shpn{81.3}{67.9}{22.3}{1} \\
      Vanilla OPD \cite{rethink_opd1}
      & \shpn{31.5}{27.5}{7.3}{1} & \shpn{23.5}{22.7}{4.0}{1} & \shpn{20.0}{21.3}{5.2}{1.3}
      & \shpn{71.1}{61.6}{10.3}{1} & \shpn{31.7}{35.0}{1.7}{3.3} & \shpn{80.4}{78.4}{4.0}{1}
      & \shpn{43.0}{41.1}{5.1}{1} & \shpn{76.4}{67.9}{22.3}{1} \\
      Evol-Instruct \cite{xu2024wizardlm}
      & \shpn{30.0}{27.5}{7.3}{1} & \shpn{21.2}{22.7}{4.0}{1} & \shpn{21.0}{21.3}{5.2}{1.3}
      & \shpn{63.9}{61.6}{10.3}{1} & \shpbn{37.9}{35.0}{1.7}{3.3} & \shpn{78.2}{78.4}{4.0}{1}
      & \shpn{42.1}{41.1}{5.1}{1} & \shpn{72.0}{67.9}{22.3}{1} \\
      R-Zero \cite{huang2025rzero}
      & \shpn{29.6}{27.5}{7.3}{1} & \shpn{21.7}{22.7}{4.0}{1} & \shpn{20.8}{21.3}{5.2}{1.3}
      & \shpn{65.0}{61.6}{10.3}{1} & \shpn{37.5}{35.0}{1.7}{3.3} & \shpn{79.5}{78.4}{4.0}{1}
      & \shpn{42.3}{41.1}{5.1}{1} & \shpn{73.4}{67.9}{22.3}{1} \\
      \textbf{\method}
      & \shpbn{34.8}{27.5}{7.3}{1} & \shpn{26.0}{22.7}{4.0}{1} & \shpbn{26.5}{21.3}{5.2}{1.3}
      & \shpbn{71.9}{61.6}{10.3}{1} & \shpn{35.4}{35.0}{1.7}{3.3} & \shpbn{82.4}{78.4}{4.0}{1}
      & \shpbn{46.2}{41.1}{5.1}{1} & \shpbn{90.2}{67.9}{22.3}{1} \\
      \midrule

      \rowcolor{groupgray}
      \multicolumn{9}{@{}l@{}}{\hspace{3pt}\textbf{DeepSeek-R1-Distill-Qwen-1.5B}\quad{\color{deltaGray}Teacher: JustRL-DeepSeek-1.5B}} \\
      Student
      & 31.0 & 26.0 & 19.4
      & 65.8 & 29.6 & 60.9
      & 38.8 & 0.0 \\
      \tv{Teacher}$^{\dagger}$
      & \tv{44.8} & \tv{32.9} & \tv{32.3}
      & \tv{85.5} & \tv{34.6} & \tv{79.8}
      & \tv{51.7} & \tv{100.0} \\
      \midrule
      \makecell[l]{Student-Generated\\OPD}
      & 37.9 & 30.4 & 26.9
      & 77.8 & 34.0 & 82.2
      & 48.2 & 73.2 \\
      KL-Filtered OPD
      & \shpn{37.5}{37.9}{7.5}{0.4} & \shpn{29.2}{30.4}{5.4}{1.2} & \shpn{29.0}{26.9}{2.3}{1}
      & \shpn{79.1}{77.8}{5.8}{1} & \shpn{35.2}{34.0}{1.8}{1} & \shpn{81.7}{82.2}{1}{0.7}
      & \shpn{48.6}{48.2}{3.4}{1} & \shpn{76.4}{73.2}{26.7}{1} \\
      Vanilla OPD \cite{rethink_opd1}
      & \shpbn{45.4}{37.9}{7.5}{0.4} & \shpn{32.7}{30.4}{5.4}{1.2} & \shpbn{29.2}{26.9}{2.3}{1}
      & \shpn{78.1}{77.8}{5.8}{1} & \shpbn{35.8}{34.0}{1.8}{1} & \shpn{81.5}{82.2}{1}{0.7}
      & \shpn{50.5}{48.2}{3.4}{1} & \shpn{90.7}{73.2}{26.7}{1} \\
      Evol-Instruct \cite{xu2024wizardlm}
      & \shpn{41.5}{37.9}{7.5}{0.4}
      & \shpn{32.5}{30.4}{5.4}{1.2}
      & \shpn{28.1}{26.9}{2.3}{1}
      & \shpn{79.5}{77.8}{5.8}{1}
      & \shpn{35.0}{34.0}{1.8}{1}
      & \shpbn{82.3}{82.2}{1}{0.7}
      & \shpn{49.8}{48.2}{3.4}{1}
      & \shpn{85.6}{73.2}{26.7}{1} \\

      R-Zero \cite{huang2025rzero}
      & \shpn{42.1}{37.9}{7.5}{0.4}
      & \shpn{32.9}{30.4}{5.4}{1.2}
      & \shpn{28.3}{26.9}{2.3}{1}
      & \shpn{80.0}{77.8}{5.8}{1}
      & \shpn{35.2}{34.0}{1.8}{1}
      & \shpn{81.7}{82.2}{1}{0.7}
      & \shpn{50.0}{48.2}{3.4}{1}
      & \shpn{87.2}{73.2}{26.7}{1} \\
      \textbf{\method}
      & \shpn{45.0}{37.9}{7.5}{0.4} & \shpbn{35.8}{30.4}{5.4}{1.2} & \shpbn{29.2}{26.9}{2.3}{1}
      & \shpbn{83.6}{77.8}{5.8}{1} & \shpn{34.0}{34.0}{1.8}{1} & \shpn{82.2}{82.2}{1}{0.7}
      & \shpbn{51.6}{48.2}{3.4}{1} & \shpbn{99.9}{73.2}{26.7}{1} \\
      \midrule

      \rowcolor{groupgray}
      \multicolumn{9}{@{}l@{}}{\hspace{3pt}\textbf{DeepSeek-R1-Distill-Qwen-7B}\quad{\color{deltaGray}Teacher: Skywork-OR1-Math-7B}} \\
      Student
      & 46.5 & 36.7 & 44.0
      & 84.1 & 49.8 & 87.3
      & 58.1 & 0.0 \\
      \tv{Teacher}$^{\dagger}$
      & \tv{62.7} & \tv{42.9} & \tv{50.8}
      & \tv{89.7} & \tv{59.0} & \tv{91.0}
      & \tv{66.0} & \tv{100.0} \\
      \midrule
      \makecell[l]{Student-Generated\\OPD}
      & 53.1 & 41.0 & 45.0
      & 84.4 & 53.3 & 88.0
      & 60.8 & 34.4 \\
      KL-Filtered OPD
      & \shpn{53.5}{53.1}{4.4}{1} & \shpn{40.2}{41.0}{3.8}{0.8} & \shpn{49.0}{45.0}{7.5}{1}
      & \shpn{87.0}{84.4}{4.8}{1} & \shpbn{55.0}{53.3}{1.7}{0.4} & \shpn{89.3}{88.0}{1.9}{1}
      & \shpn{62.3}{60.8}{3.7}{1} & \shpn{53.7}{34.4}{46.9}{1} \\
      Vanilla OPD \cite{rethink_opd1}
      & \shpn{54.8}{53.1}{4.4}{1} & \shpn{41.5}{41.0}{3.8}{0.8} & \shpn{50.0}{45.0}{7.5}{1}
      & \shpn{88.8}{84.4}{4.8}{1} & \shpn{52.9}{53.3}{1.7}{0.4} & \shpn{89.5}{88.0}{1.9}{1}
      & \shpn{62.9}{60.8}{3.7}{1} & \shpn{61.0}{34.4}{46.9}{1} \\
      Evol-Instruct \cite{xu2024wizardlm}
      & \shpn{54.6}{53.1}{4.4}{1}
      & \shpn{41.5}{41.0}{3.8}{0.8}
      & \shpn{49.0}{45.0}{7.5}{1}
      & \shpn{86.8}{84.4}{4.8}{1}
      & \shpn{54.0}{53.3}{1.7}{0.4}
      & \shpbn{90.4}{88.0}{1.9}{1}
      & \shpn{62.7}{60.8}{3.7}{1}
      & \shpn{58.3}{34.4}{46.9}{1} \\

      R-Zero \cite{huang2025rzero}
      & \shpn{55.0}{53.1}{4.4}{1}
      & \shpn{42.1}{41.0}{3.8}{0.8}
      & \shpn{49.6}{45.0}{7.5}{1}
      & \shpn{87.0}{84.4}{4.8}{1}
      & \shpn{52.5}{53.3}{1.7}{0.4}
      & \shpn{89.0}{88.0}{1.9}{1}
      & \shpn{62.5}{60.8}{3.7}{1}
      & \shpn{55.8}{34.4}{46.9}{1} \\
      \textbf{\method}
      & \shpbn{57.5}{53.1}{4.4}{1} & \shpbn{44.8}{41.0}{3.8}{0.8} & \shpbn{52.5}{45.0}{7.5}{1}
      & \shpbn{89.2}{84.4}{4.8}{1} & \shpn{53.3}{53.3}{1.7}{0.4} & \shpn{89.9}{88.0}{1.9}{1}
      & \shpbn{64.5}{60.8}{3.7}{1} & \shpbn{81.3}{34.4}{46.9}{1} \\
      \bottomrule
    \end{tabular}%
  }

  \vspace{3pt}
  \begin{minipage}{\textwidth}
    \scriptsize
    $^{\dagger}$ Teachers are reference models.
    Recovery: $(\mathrm{Avg.}_{\mathrm{method}}-\mathrm{Avg.}_{\mathrm{student}})/(\mathrm{Avg.}_{\mathrm{teacher}}-\mathrm{Avg.}_{\mathrm{student}})\times100\%$, computed from unrounded averages and rounded to one decimal place.
  \end{minipage}
\end{table*}%
\subsection{Optimizing the synthetic data distribution}
\label{sec:synthetic_optimization}

We initialize $S_\theta$ from the student to use its existing generative capabilities as the starting point for question construction. Shared initialization supplies a prior for generation; the reward determines which questions are favored for the target student--teacher pair. Throughout this stage, the scoring student and teacher remain frozen.

For each question-generation instruction $x$, we sample a group of $g$ candidate questions from $S_\theta(\cdot\mid x)$ under a question budget $L_q$. The synthesis and student-solving prompt templates are provided in Appendix~\ref{app:prompts}. For each candidate $q_i$, the frozen student generates one response, and the teacher is evaluated on the same response prefixes. These predictions provide the token scores and response-length factor. We compute the batch-diversity factor across the full scoring batch and assign each candidate its reward from Eq.~\eqref{eq:reward}.

We optimize the synthesizer with GRPO, normalizing rewards within each group of questions sampled from the same instruction. The resulting group-relative advantage is applied to the generated tokens of the corresponding question. Thus, the optimization updates the policy that generates questions; the student responses are used only to score those questions. Operationally, the question-level reward is attached to the final generated question token as a sequence-level outcome reward. This reward closes the feedback loop for question generation: it evaluates a candidate through the supervision exposed by a frozen student--teacher pair and uses that signal to improve the synthesizer. Candidate scoring requires neither reference answers nor task-specific correctness rewards, and does not involve training a student on each candidate. It therefore provides a distributional surrogate for question value before downstream distillation.

\subsection{Synthetic data curation}
\label{sec:bank}

After synthesizer optimization, we sample $N$ raw outputs from the selected checkpoint.
 We first filter malformed or incomplete outputs. When a valid question is followed by model self-talk, we retain the question by truncating at the earliest meta-commentary marker. We otherwise reject outputs with repeated $n$-grams, low lexical diversity, unbalanced mathematical delimiters, answer leakage, missing context, or references to absent figures.

We then deduplicate the remaining questions in two passes.
 Token-set Jaccard matching removes lexical near-duplicates whose similarity reaches a threshold $\tau_{\mathrm{lex}}$. A subsequent greedy semantic pass uses a frozen Qwen3-Embedding-4B model~\citep{qwen3-embedding}: a question is retained only if the cosine similarity of its embedding to that of every previously retained question is below $\tau_{\mathrm{sem}}$. This embedding model is used only for post-generation deduplication and does not participate in synthesizer training or reward computation.

The retained questions form a prompt-only synthetic bank. Any placeholder answers stored for file-format compatibility are ignored during OPD; all training responses are sampled from the student.

\subsection{On-policy distillation}
\label{sec:opd}

We use the curated question bank for candidate-restricted OPD, starting from the original student parameters. For each question, the current rollout policy $\pi_{S,\mathrm{old}}$ samples $n_{\mathrm{roll}}$ responses under the response budget $L_{\max}$ used for question scoring. The teacher remains frozen and is evaluated on the student-generated prefixes.

At each response position $t$, let $p_{S,\mathrm{old}}(v)=\pi_{S,\mathrm{old}}(v\mid c_t)$ and $p_T(v)=\pi_T(v\mid c_t)$. We construct $A_t=\Topk(p_{S,\mathrm{old}})$ with the same $K$ as in \S\ref{sec:reward} and define the student probabilities normalized over this set:
\begin{equation}
  \bar p_{S,\mathrm{old}}^{A}(v)
  =
  \frac{p_{S,\mathrm{old}}(v)}
       {\sum_{u\in A_t}p_{S,\mathrm{old}}(u)},
  \qquad v\in A_t.
\end{equation}
The superscript $A$ distinguishes this normalization from the union-support normalization used to score questions in \S\ref{sec:reward}.
\begin{table*}[!t]
  \centering
  \caption{
    Results across six code benchmarks for three student--teacher pairs.
    Within each block and metric, shading is relative to Student-Generated OPD
    (white): \textcolor{OPDBlue}{blue} indicates improvement and
    \textcolor{OPDRed}{red} indicates degradation. Positive and negative changes
    are normalised separately across OPD methods and mapped nonlinearly
    (exponent 1.5) to muted palettes.
    Bold indicates the best result among OPD methods, including ties.
    The baselines compare real-bank OPD, untrained student-generated prompts,
    KL-filtered selection, and two additional code-generation methods.
  }
  \label{tab:code_main}

  \footnotesize
  \setlength{\tabcolsep}{6pt}
  \renewcommand{\arraystretch}{1.2}

  \resizebox{\textwidth}{!}{%
    \begin{tabular}{@{}l cccccc @{\hspace{14pt}} c c@{}}
      \toprule
      Method & LCB & MBPP+ & APPS-Intro & APPS-Interv. & APPS-Comp.
      & CodeContests & Avg. & Recovery (\%) \\
      \midrule

      \rowcolor{groupgray}
      \multicolumn{9}{@{}l@{}}{\hspace{3pt}\textbf{Qwen3-4B (Non-thinking)}\quad{\color{deltaGray}Teacher: Qwen3-4B-Instruct-2507}} \\
      Student
      & 21.6 & 68.8 & 46.4 & 19.1 & 8.8 & 7.3 & 28.7 & 0.0 \\
      \tv{Teacher}$^{\dagger}$
      & \tv{33.7} & \tv{72.0} & \tv{65.6} & \tv{41.0} & \tv{19.7} & \tv{18.8}
      & \tv{41.8} & \tv{100.0} \\
      \midrule
      \makecell[l]{Student-Generated\\OPD}
      & 30.2 & 67.2 & 55.8 & 27.8 & 10.4 & 9.7 & 33.5 & 36.9 \\
      KL-Filtered OPD
      & \shpn{31.1}{30.2}{5.7}{1} & \shpn{67.7}{67.2}{1.3}{1} & \shpn{58.2}{55.8}{8.2}{1}
      & \shpn{30.0}{27.8}{8.2}{0.2} & \shpn{11.5}{10.4}{5.5}{1} & \shpn{11.5}{9.7}{7.3}{1}
      & \shpn{35.0}{33.5}{6.1}{1} & \shpn{48.2}{36.9}{46.0}{1} \\
      Vanilla OPD \cite{rethink_opd1}
      & \shpn{31.0}{30.2}{5.7}{1} & \shpn{68.5}{67.2}{1.3}{1} & \shpn{57.1}{55.8}{8.2}{1}
      & \shpn{27.6}{27.8}{8.2}{0.2} & \shpn{10.7}{10.4}{5.5}{1} & \shpn{12.1}{9.7}{7.3}{1}
      & \shpn{34.5}{33.5}{6.1}{1} & \shpn{44.4}{36.9}{46.0}{1} \\
      Evol-Instruct \cite{xu2024wizardlm}
      & \shpn{30.3}{30.2}{5.7}{1} & \shpn{68.5}{67.2}{1.3}{1} & \shpn{56.8}{55.8}{8.2}{1}
      & \shpn{27.8}{27.8}{8.2}{0.2} & \shpn{10.0}{10.4}{5.5}{1} & \shpn{10.3}{9.7}{7.3}{1}
      & \shpn{34.0}{33.5}{6.1}{1} & \shpn{40.2}{36.9}{46.0}{1} \\
      SQLM \cite{chen2025selfquestioning}
      & \shpn{31.3}{30.2}{5.7}{1} & \shpbn{69.8}{67.2}{1.3}{1} & \shpn{57.7}{55.8}{8.2}{1}
      & \shpn{28.9}{27.8}{8.2}{0.2} & \shpn{10.9}{10.4}{5.5}{1} & \shpn{10.9}{9.7}{7.3}{1}
      & \shpn{34.9}{33.5}{6.1}{1} & \shpn{47.6}{36.9}{46.0}{1} \\
      \textbf{\method}
      & \shpbn{35.9}{30.2}{5.7}{1} & \shpn{68.5}{67.2}{1.3}{1} & \shpbn{64.0}{55.8}{8.2}{1}
      & \shpbn{36.0}{27.8}{8.2}{0.2} & \shpbn{15.9}{10.4}{5.5}{1} & \shpbn{17.0}{9.7}{7.3}{1}
      & \shpbn{39.6}{33.5}{6.1}{1} & \shpbn{82.9}{36.9}{46.0}{1} \\
      \midrule

      \rowcolor{groupgray}
      \multicolumn{9}{@{}l@{}}{\hspace{3pt}\textbf{Qwen3-4B (Non-thinking)}\quad{\color{deltaGray}Teacher: Qwen3-30B-A3B-Instruct-2507}} \\
      Student
      & 21.6 & 68.8 & 46.4 & 19.1 & 8.8 & 7.3 & 28.7 & 0.0 \\
      \tv{Teacher}$^{\dagger}$
      & \tv{42.3} & \tv{75.9} & \tv{73.4} & \tv{51.2} & \tv{30.4} & \tv{26.7}
      & \tv{50.0} & \tv{100.0} \\
      \midrule
      \makecell[l]{Student-Generated\\OPD}
      & 28.0 & 66.9 & 50.5 & 20.0 & 5.5 & 6.1 & 29.5 & 3.9 \\
      KL-Filtered OPD
      & \shpn{29.4}{28.0}{3.7}{1} & \shpn{66.9}{66.9}{1.9}{1} & \shpn{53.4}{50.5}{7.8}{1}
      & \shpn{25.1}{20.0}{11.6}{1} & \shpn{8.4}{5.5}{5.5}{1.3} & \shpn{9.7}{6.1}{7.8}{1.3}
      & \shpn{32.2}{29.5}{6.2}{1} & \shpn{16.3}{3.9}{28.9}{1} \\
      Vanilla OPD \cite{rethink_opd1}
      & \shpn{30.6}{28.0}{3.7}{1} & \shpbn{68.8}{66.9}{1.9}{1} & \shpn{51.3}{50.5}{7.8}{1}
      & \shpn{21.3}{20.0}{11.6}{1} & \shpn{4.2}{5.5}{5.5}{1.3} & \shpn{4.8}{6.1}{7.8}{1.3}
      & \shpn{30.2}{29.5}{6.2}{1} & \shpn{7.0}{3.9}{28.9}{1} \\
      Evol-Instruct \cite{xu2024wizardlm}
      & \shpn{28.8}{28.0}{3.7}{1}
      & \shpn{67.1}{66.9}{1.9}{1}
      & \shpn{51.9}{50.5}{7.8}{1}
      & \shpn{22.5}{20.0}{11.6}{1}
      & \shpn{5.0}{5.5}{5.5}{1.3}
      & \shpn{5.0}{6.1}{7.8}{1.3}
      & \shpn{30.1}{29.5}{6.2}{1}
      & \shpn{6.5}{3.9}{28.9}{1} \\

      SQLM \cite{chen2025selfquestioning}
      & \shpn{29.9}{28.0}{3.7}{1}
      & \shpn{67.5}{66.9}{1.9}{1}
      & \shpn{53.5}{50.5}{7.8}{1}
      & \shpn{23.7}{20.0}{11.6}{1}
      & \shpn{5.7}{5.5}{5.5}{1.3}
      & \shpn{6.3}{6.1}{7.8}{1.3}
      & \shpn{31.1}{29.5}{6.2}{1}
      & \shpn{11.4}{3.9}{28.9}{1} \\
      \textbf{\method}
      & \shpbn{31.7}{28.0}{3.7}{1} & \shpn{67.5}{66.9}{1.9}{1} & \shpbn{58.3}{50.5}{7.8}{1}
      & \shpbn{31.6}{20.0}{11.6}{1} & \shpbn{11.0}{5.5}{5.5}{1.3} & \shpbn{13.9}{6.1}{7.8}{1.3}
      & \shpbn{35.7}{29.5}{6.2}{1} & \shpbn{32.8}{3.9}{28.9}{1} \\
      \midrule

      \rowcolor{groupgray}
      \multicolumn{9}{@{}l@{}}{\hspace{3pt}\textbf{Qwen2.5-Coder-3B-Instruct}\quad{\color{deltaGray}Teacher: Qwen2.5-Coder-32B-Instruct}} \\
      Student
      & 13.7 & 64.0 & 28.7 & 6.6 & 3.0 & 2.4 & 19.7 & 0.0 \\
      \tv{Teacher}$^{\dagger}$
      & \tv{27.9} & \tv{76.7} & \tv{51.8} & \tv{18.8} & \tv{11.1} & \tv{5.5}
      & \tv{32.0} & \tv{100.0} \\
      \midrule
      \makecell[l]{Student-Generated\\OPD}
      & 14.5 & 64.3 & 29.0 & 6.3 & 2.8 & 1.8 & 19.8 & 0.4 \\
      KL-Filtered OPD
      & \shpn{15.3}{14.5}{2.5}{1} & \shpn{64.8}{64.3}{0.8}{0.3} & \shpn{30.0}{29.0}{2.6}{1}
      & \shpn{6.8}{6.3}{1.3}{1} & \shpn{3.4}{2.8}{0.9}{1} & \shpbn{2.4}{1.8}{0.6}{1.2}
      & \shpn{20.5}{19.8}{1.4}{1} & \shpn{5.9}{0.4}{11.9}{1} \\
      Vanilla OPD \cite{rethink_opd1}
      & \shpn{16.2}{14.5}{2.5}{1} & \shpn{64.0}{64.3}{0.8}{0.3} & \shpn{30.7}{29.0}{2.6}{1}
      & \shpn{6.5}{6.3}{1.3}{1} & \shpn{2.9}{2.8}{0.9}{1} & \shpn{0.6}{1.8}{0.6}{1.2}
      & \shpn{20.2}{19.8}{1.4}{1} & \shpn{3.4}{0.4}{11.9}{1} \\
      Evol-Instruct \cite{xu2024wizardlm}
      & \shpn{15.5}{14.5}{2.5}{1}
      & \shpn{64.6}{64.3}{0.8}{0.3}
      & \shpn{29.8}{29.0}{2.6}{1}
      & \shpn{6.9}{6.3}{1.3}{1}
      & \shpn{3.2}{2.8}{0.9}{1}
      & \shpn{1.6}{1.8}{0.6}{1.2}
      & \shpn{20.3}{19.8}{1.4}{1}
      & \shpn{4.4}{0.4}{11.9}{1} \\

      SQLM \cite{chen2025selfquestioning}
      & \shpn{15.8}{14.5}{2.5}{1}
      & \shpn{64.8}{64.3}{0.8}{0.3}
      & \shpn{30.4}{29.0}{2.6}{1}
      & \shpn{7.1}{6.3}{1.3}{1}
      & \shpn{3.4}{2.8}{0.9}{1}
      & \shpn{1.8}{1.8}{0.6}{1.2}
      & \shpn{20.6}{19.8}{1.4}{1}
      & \shpn{6.7}{0.4}{11.9}{1} \\
      \textbf{\method}
      & \shpbn{17.0}{14.5}{2.5}{1} & \shpbn{65.1}{64.3}{0.8}{0.3} & \shpbn{31.6}{29.0}{2.6}{1}
      & \shpbn{7.6}{6.3}{1.3}{1} & \shpbn{3.7}{2.8}{0.9}{1} & \shpbn{2.4}{1.8}{0.6}{1.2}
      & \shpbn{21.2}{19.8}{1.4}{1} & \shpbn{12.3}{0.4}{11.9}{1} \\

      \bottomrule
    \end{tabular}%
  }

  \vspace{3pt}
  \begin{minipage}{\textwidth}
    \scriptsize
    $^{\dagger}$ Teachers are reference models.
    Recovery:
    $(\mathrm{Avg.}_{\mathrm{method}}-\mathrm{Avg.}_{\mathrm{student}})
    /(\mathrm{Avg.}_{\mathrm{teacher}}-\mathrm{Avg.}_{\mathrm{student}})\times100\%$,
    computed from unrounded averages and rounded to one decimal place.
  \end{minipage}
\end{table*}%
 Each candidate receives the fixed advantage
\begin{equation}
  r_t(v)
  =
  -\bar p_{S,\mathrm{old}}^{A}(v)
  \left[
    \log p_{S,\mathrm{old}}(v)
    -
    \log p_T(v)
  \right],
  \qquad v\in A_t.
  \label{eq:opdreward}
\end{equation}
The log-probabilities in brackets are full-vocabulary probabilities; only the weighting factor $\bar p_{S,\mathrm{old}}^{A}$ is normalized over $A_t$. The candidate sets and advantages are computed from the rollout policy and held fixed during the subsequent update.

Let $p_{S,\phi}(v)=\pi_{S,\phi}(v\mid c_t)$ denote the trainable student's probability and define
\begin{equation}
  \rho_t(v)
  =
  \frac{p_{S,\phi}(v)}
       {p_{S,\mathrm{old}}(v)}.
\end{equation}
\begin{figure*}[!t]
  \centering
  \includegraphics[width=\textwidth]{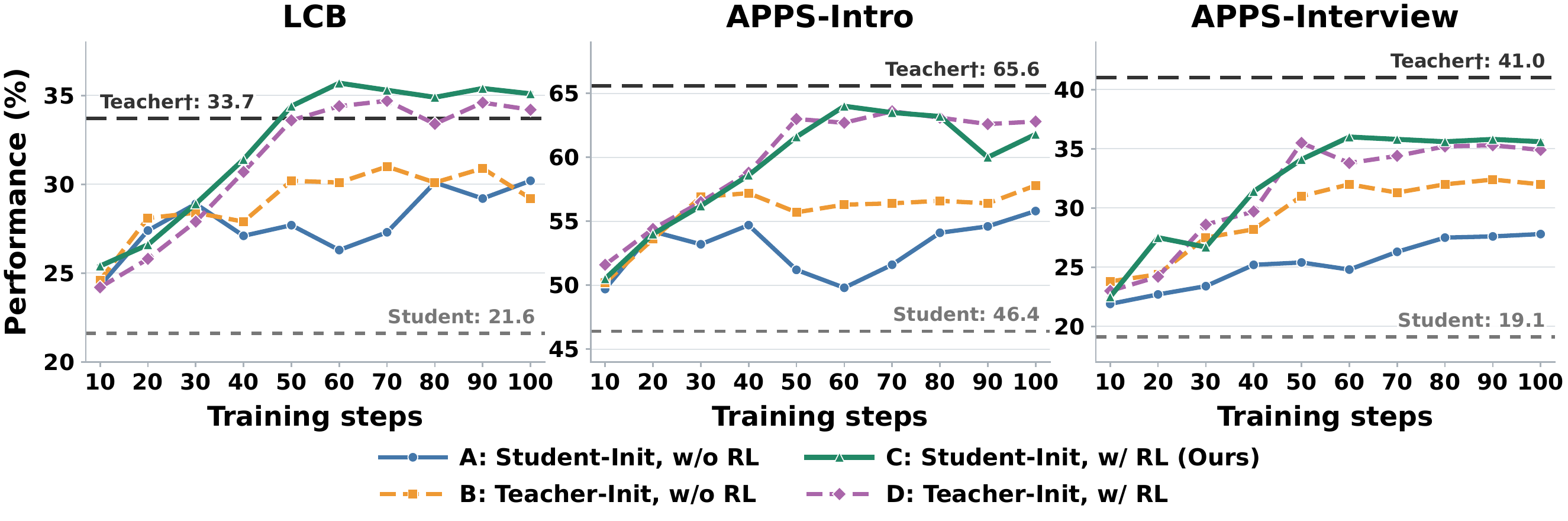}
  \caption{Downstream performance during training for student- and
    teacher-initialized question synthesizers, with and without pair-specific RL,
    using Qwen3-4B (non-thinking) as the student and
    Qwen3-4B-Instruct-2507 as the teacher, evaluated on LCB, APPS-Intro, and
    APPS-Interview. Dashed lines show the
    student and teacher reference scores. RL improves both initializations;
    student initialization with RL is competitive across benchmarks and exceeds
    the teacher reference on LCB, while remaining below it on APPS-Intro and
  APPS-Interview.}
  \label{fig:synthesizer_source}
\end{figure*}%

\begin{table}[t]
  \centering
  \caption{
    Results across the three $\tau^2$-Bench domains.
    Within each metric, shading is relative to Student-Generated OPD
    (white): \textcolor{OPDBlue}{blue} indicates improvement and
    \textcolor{OPDRed}{red} indicates degradation. Positive and negative
    changes are normalised separately across OPD methods and mapped
    nonlinearly (exponent 1.5) to muted palettes.
    Bold indicates the best result among OPD methods, including ties.
    The three baselines isolate the effects of synthesizer training
    and learning-value-based question generation.
  }
  \label{tab:tau2_results}

  \footnotesize
  \setlength{\tabcolsep}{4pt}
  \renewcommand{\arraystretch}{1.2}

  \resizebox{\columnwidth}{!}{%
  \begin{tabular}{@{}l ccc @{\hspace{8pt}} c c@{}}
    \toprule
    Method & Airline & Retail & Telecom & Avg.
    & \makecell{Recovery\\(\%)} \\
    \midrule

    \rowcolor{groupgray}
    \multicolumn{6}{@{}l@{}}{%
      \makecell[l]{%
        \textbf{Qwen3-4B (Non-thinking)}\\
        \scriptsize\textcolor{deltaGray}{%
          Teacher: Qwen3-4B-Thinking-2507 (Non-thinking)}%
      }%
    } \\

    Student
    & 15.0 & 17.5 & 10.6 & 14.4 & 0.0 \\

    \tv{Teacher}$^{\dagger}$
    & \tv{47.5} & \tv{48.1} & \tv{34.4}
    & \tv{43.3} & \tv{100.0} \\
    \midrule

    \makecell[l]{Student-Generated\\OPD}
    & 42.5 & 42.5 & 26.9 & 37.3 & 79.2 \\

    KL-Filtered OPD
    & \shpn{45.0}{42.5}{3.8}{1}
    & \shpn{48.8}{42.5}{6.9}{1}
    & \shpn{28.8}{26.9}{3.1}{1}
    & \shpn{40.9}{37.3}{4.6}{1}
    & \shpn{91.7}{79.2}{16.0}{1} \\

    Vanilla OPD
    & \shpn{45.0}{42.5}{3.8}{1}
    & \shpn{46.9}{42.5}{6.9}{1}
    & \shpn{29.4}{26.9}{3.1}{1}
    & \shpn{40.4}{37.3}{4.6}{1}
    & \shpn{90.0}{79.2}{16.0}{1} \\

    \textbf{\method}
    & \shpbn{46.3}{42.5}{3.8}{1}
    & \shpbn{49.4}{42.5}{6.9}{1}
    & \shpbn{30.0}{26.9}{3.1}{1}
    & \shpbn{41.9}{37.3}{4.6}{1}
    & \shpbn{95.2}{79.2}{16.0}{1} \\
    \bottomrule
  \end{tabular}%
  }

  \par\vspace{3pt}
  \begin{minipage}{\columnwidth}
    \scriptsize
    $^{\dagger}$ The teacher is a reference model.
    Recovery:
    $(\mathrm{Avg.}_{\mathrm{method}}-\mathrm{Avg.}_{\mathrm{student}})
    /(\mathrm{Avg.}_{\mathrm{teacher}}-\mathrm{Avg.}_{\mathrm{student}})
    \times100\%$,
    computed from unrounded averages and rounded to one decimal place.
  \end{minipage}
\end{table}

We use the candidate advantages directly in a PPO-style clipped objective \cite{ppo}:
\begin{equation}
  \begin{aligned}
    \widehat r_t(v)
    ={}&
    \min\!\left\{
      \rho_t(v)r_t(v),
      \right.\\
      &\left.
      \mathrm{clip}\!\left(
        \rho_t(v),1-\varepsilon,1+\varepsilon
      \right)r_t(v)
    \right\},\\
    \mathcal{L}(\phi)
    ={}&
    -\mathbb{E}\!\left[
      \frac{
        \sum_t m_t
        \sum_{v\in A_t}\widehat r_t(v)
      }{
        \sum_t m_t
      }
    \right],
  \end{aligned}
\end{equation}
where $m_t$ masks valid response positions and the expectation is over sampled questions and student responses. Gradients flow through $p_{S,\phi}$, while the rollout probabilities, candidate sets, and advantages are treated as constants.
 We use a single inner epoch per rollout batch, without a value function, return accumulation, group normalization of these advantages, or a KL anchor to a reference model.

At the beginning of an update, the trainable student matches the rollout policy and $\rho_t(v)=1$; the ratios can change as optimization proceeds. The advantages provide a student-probability-weighted, reverse-KL-style learning signal restricted to the student's top-$K$ candidates. We do not identify this clipped, candidate-restricted objective with exact full-vocabulary reverse-KL minimization.

The question-generation reward and the downstream update serve different roles. The forward-KL discrepancy in Eq.\eqref{eq:div} identifies teacher-preferred candidates that the scoring student underweights, while compatibility measures how much teacher probability mass lies within the candidate set used for training. Their combination scores the supervision a question exposes without reproducing the downstream loss or directly estimating its gradient magnitude. We hold the downstream OPD procedure fixed across prompt-source comparisons to assess the effect of the constructed question distribution.

\section{Experiments}
\label{sec:experiments}

\subsection{Models}
\label{sec:models}

\paragraph{Mathematical reasoning.} We evaluate three student--teacher pairs spanning different model sizes and reasoning modes. The first pair uses Qwen3-1.7B \cite{qwen3} in non-thinking mode as the student and Qwen3-4B-Non-Thinking-RL-Math \cite{qwen3-4b-non-thinking-rl-math} as the teacher. The other two pairs use DeepSeek-R1-Distill-Qwen-1.5B and DeepSeek-R1-Distill-Qwen-7B~\citep{deepseek-r1}, paired with JustRL-DeepSeek-1.5B~\citep{justrl} and Skywork-OR1-Math-7B~\citep{skywork}, respectively. These pairs cover both non-thinking and reasoning models, including settings where the student and teacher have the same parameter count but differ in post-training.

\paragraph{Code generation.} We evaluate Qwen3-4B in non-thinking mode with two teachers: Qwen3-4B-Instruct-2507 and Qwen3-30B-A3B-Instruct-2507~\citep{qwen3}. This comparison examines distillation from different teachers while keeping the initial student fixed. We additionally evaluate Qwen2.5-Coder-3B-Instruct with Qwen2.5-Coder-32B-Instruct~\citep{qwen2.5-coder}, extending the evaluation to a code-specialized model family.

\paragraph{Agentic tasks and model roles.} For agentic tasks, we use Qwen3-4B as the student and Qwen3-4B-Thinking-2507 as the teacher, both configured in non-thinking mode. In our default setup, the question synthesizer is initialized as an independent copy of the corresponding student. During synthesizer optimization, the scoring student and teacher remain frozen; downstream OPD starts from the original student parameters and updates only the student. We separately examine teacher-initialized synthesizers in the synthesizer-initialization ablation. Exact checkpoint identifiers and inference configurations are provided in the appendix.

\paragraph{Auxiliary models.} For agentic tasks, we use GPT-4o-mini \cite{openai2024gpt4omini} as the user simulator and judge. For semantic deduplication of the generated question bank, we use a frozen Qwen3-Embedding-4B model~\citep{qwen3-embedding} only after synthesizer training. It is not used to compute training-time diversity rewards or optimize the synthesizer. These auxiliary models remain fixed throughout our experiments.

\subsection{Baselines}
\label{sec:baselines}

We compare \method with the following baselines and reference models: (1) \textbf{Student and Teacher}: the initial student before distillation and its corresponding teacher, which establish the starting performance and the reference for teacher-gap recovery; (2) \textbf{Vanilla OPD} \cite{rethink_opd1}: on-policy distillation using the original question bank for each domain, DAPO-17k-Math~\citep{dapo} for mathematical reasoning, OpenCodeReasoning~\citep{code-dataset} for code generation, and the official $\tau^2$-Bench training set~\citep{tau2} for agentic tasks---where the student generates responses and the teacher provides token-level supervision on student-generated prefixes, serving as the standard real-data reference; (3) \textbf{Student-Generated OPD}: OPD using questions sampled directly from the initial student without synthesizer optimization, evaluating whether explicitly training the question-generation policy provides benefits beyond the student's existing generation capability; (4) \textbf{KL-Filtered OPD}: OPD using the 20\% of real questions with the highest teacher--student KL disagreement, testing whether retaining high-disagreement questions alone provides effective distillation data;
\begin{figure}[!t]
  \centering
  \includegraphics[width=0.49\textwidth]{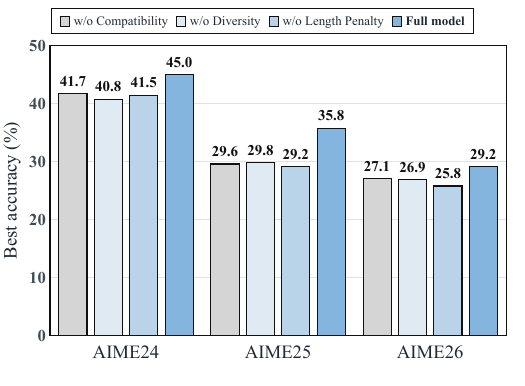}
  \caption{Reward-component ablation on AIME24, AIME25, and AIME26. Numbers above bars indicate accuracy (\%). Removing compatibility, the
    training-time diversity term, or the explicit length penalty reduces
  accuracy on all three benchmarks.}
  \label{fig:reward_ablation}
\end{figure}%

\begin{table}[t]
  \centering
  \caption{Iterative training on code pair A, with student Qwen3-4B
  (non-thinking) and teacher Qwen3-30B-A3B-Instruct-2507. Avg.\ is the mean
  of LiveCodeBench v5\_v6, APPS-Intro and APPS-Interview. Recovery
  $=(\text{Avg}-\text{student})/(\text{teacher}-\text{student})$. Dashes indicate entries that are not applicable.}
  \label{tab:iterative_training}
  \footnotesize
  \setlength{\tabcolsep}{3pt}
  \begin{tabular}{@{}lccc@{}}
    \toprule
    \textbf{Training variant}
    & \textbf{Avg.}
    & \makecell{\textbf{Gain over}\\\textbf{round 1}}
    & \makecell{\textbf{Recovery}\\\textbf{(\%)}} \\ \midrule
    Student & 29.0 & - & 0.0 \\
    Teacher & 55.6 & - & 100.0 \\
    \midrule
    \method: round 1 & 40.5 & 0.0 & 43.3 \\
    \method: round 2 & 40.8 & +0.3 & 44.3 \\
    \rowcolor{methodBg}
    \method: final round ($R=3$) & 41.0 & +0.5 & 44.8 \\
    \bottomrule
  \end{tabular}
\end{table}%

 (5) \textbf{Evol-Instruct}~\citep{xu2024wizardlm}: an instruction-evolution method that iteratively rewrites seed instructions to increase complexity and diversity, which we adapt to construct mathematical and programming questions for OPD;
 (6) \textbf{R-Zero}~\citep{huang2025rzero}: a self-evolution framework in which a Challenger proposes tasks and a Solver learns to solve them, with Challenger optimization guided by feedback about the Solver's capabilities, whose task-generation procedure we use to construct mathematical questions for OPD;
 (7) \textbf{SQLM}~\citep{chen2025selfquestioning}: an asymmetric self-play framework with a topic-conditioned proposer and a solver, whose coding formulation uses generated unit tests and solver test performance to encourage non-trivial but solvable tasks, providing programming questions for downstream OPD. Evol-Instruct is evaluated in both domains, while R-Zero and SQLM are included in the mathematical and code-generation comparisons, respectively. For these question-bank comparisons, the generated or selected questions are used as prompts for the same downstream OPD procedure. We reproduce Evol-Instruct, R-Zero, and SQLM using their official implementations.

\subsection{Datasets and Evaluation Metrics}
\label{sec:datasets_metrics}

\paragraph{Mathematical reasoning.} We evaluate on six benchmarks: AIME24~\citep{aime24}, AIME25~\citep{aime25}, AIME26~\citep{brumo2025_aime26}, AMC23~\citep{amc23}, BRUMO25~\citep{brumo2025_aime26}, and MATH500~\citep{math500, math500_2}. Each AIME benchmark and BRUMO25 contains 30 problems, AMC23 contains 40, and MATH500 contains 500. These benchmarks cover competition-level mathematical reasoning across algebra, geometry, number theory, and combinatorics. We evaluate final-answer correctness using rule-based normalization and symbolic equivalence checking.

\paragraph{Code generation.} We evaluate on six benchmarks: LiveCodeBench~\citep{lcb}, MBPP+~\citep{mbpp, mbpp+}, the introductory, interview, and competition subsets of APPS~\citep{apps}, and CodeContests~\citep{codecontests}. We use the LiveCodeBench v5--v6 subset with 342 problems and MBPP+ v0.2.0 with 378 problems. For APPS, we use all 1,000 introductory and 1,000 competition test problems, together with a fixed sample of 1,000 interview test problems. CodeContests contains 165 test problems. A generated program is considered correct only if it passes every test case used by the evaluator within the time limit. For MBPP+, we apply entry-point name normalization when the required function is absent and the generated program defines exactly one top-level function.

\paragraph{Agentic tasks.} We additionally evaluate on the Airline, Retail, and Telecom domains of $\tau^2$-Bench~\citep{tau2}, using the official test splits containing 20, 40, and 40 tasks, respectively. These tasks require an agent to interact with a simulated user and invoke tools to complete domain-specific objectives. We run four trials per task and report $\mathrm{pass}^{1}$, the average single-trial task success rate under the benchmark evaluator.

\paragraph{Evaluation metrics.} Following \citet{rethink_opd1}, for mathematical reasoning and code generation, we report $\mathrm{avg}@k$, defined as the fraction of correct responses among $k$ generations per problem, averaged over all problems:
\begin{equation}
\mathrm{avg}@k = \frac{100}{Nk}\sum_{i=1}^{N}\sum_{j=1}^{k}\mathbb{I}\!\left[y_i^{(j)} \text{ is correct}\right],
\end{equation}
where $N$ is the number of problems. This metric measures average response accuracy rather than whether at least one of the $k$ responses is correct. We use $k=16$ for AIME24/25/26, AMC23, and BRUMO25; $k=4$ for MATH500 and LiveCodeBench; and $k=1$ for the remaining code benchmarks. MBPP+ uses greedy decoding. Within each domain, \textit{Avg.} is the unweighted mean of the benchmark scores. To measure how much of the initial student--teacher performance gap is recovered, we also report
\begin{equation}
\mathrm{Recovery} = 100 \times \frac{\mathrm{Avg.}_{\mathrm{method}}-\mathrm{Avg.}_{\mathrm{student}}}{\mathrm{Avg.}_{\mathrm{teacher}}-\mathrm{Avg.}_{\mathrm{student}}}.
\end{equation}
A recovery of $100\%$ indicates that the distilled student matches the teacher's average performance, while values above $100\%$ indicate that it exceeds the teacher. Recovery is computed from unrounded averages. Detailed decoding settings and evaluation procedures are provided in Appendix~\ref{sec:eval_details}.

\begin{figure}[!t]
  \centering
  \includegraphics[height=0.45\textheight]{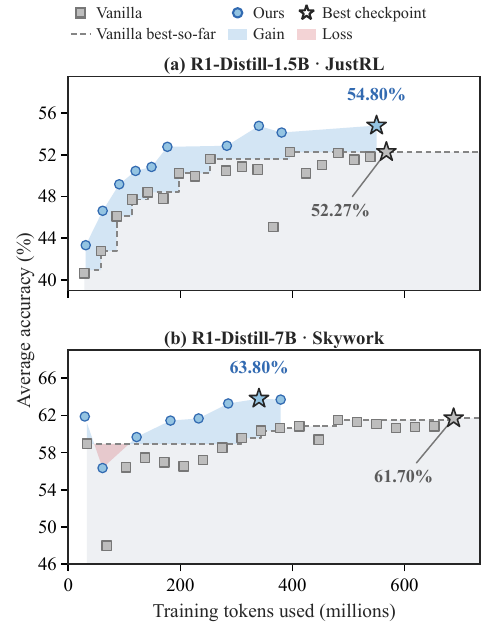}
  \caption{Average accuracy versus cumulative training tokens for (a)
    R1-Distill-1.5B on JustRL and (b) R1-Distill-7B on Skywork. Gray squares
    show Vanilla checkpoints and the dashed line its best-so-far accuracy;
    blue circles show \method checkpoints. Shaded
    regions indicate intervals where \method is above (gain) or below (loss)
    Vanilla's best-so-far reference. Average accuracy is computed over
    AIME24, AIME25, and AMC23, and therefore differs from the six-benchmark
    average in Table~\ref{tab:main}.}
  \label{fig:opd_token_efficiency}
\end{figure}

\begin{figure*}[t]
  \centering
  \includegraphics[width=\textwidth]{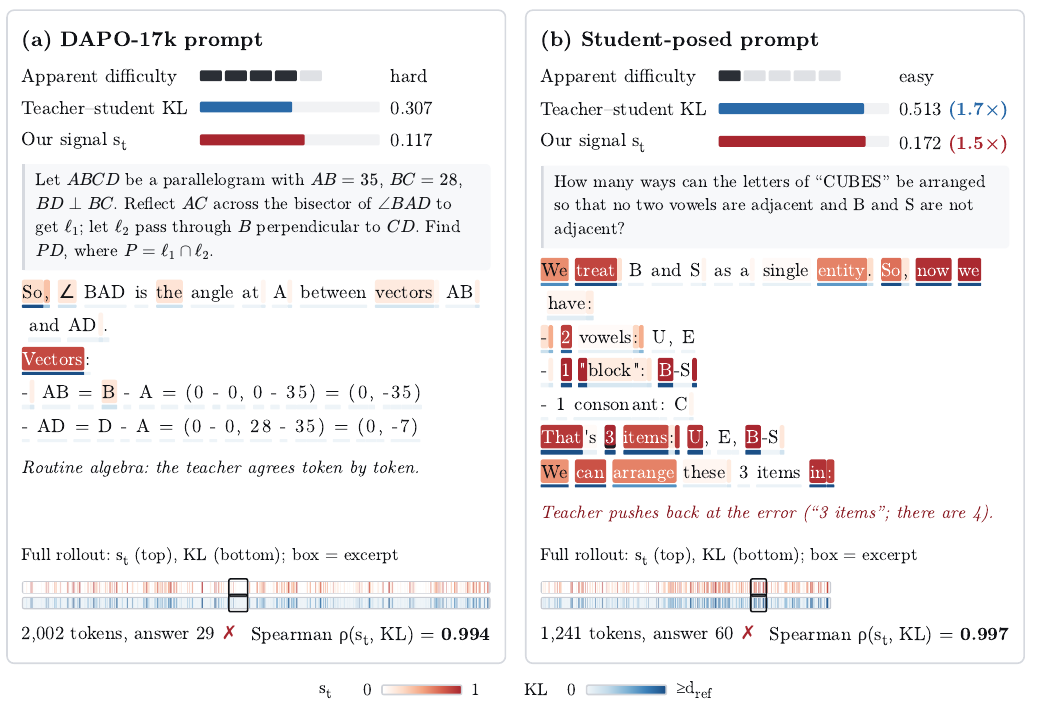}
  \caption{\textbf{Simple prompts can elicit dense distillation signals.} Comparison of (a) a real question from DAPO-17k-Math and (b) a synthetic question generated by our method. Darker token backgrounds indicate higher support-aware scores $s_t$; the strips summarize the complete rollouts, with boxes marking the displayed excerpts. Although both student answers are incorrect, the apparently simpler synthetic question elicits a higher mean $s_t$ (0.172 versus 0.117; approximately $1.5\times$) with a shorter rollout (1,241 versus 2,002 tokens), illustrating that surface complexity alone does not characterize the density of elicited supervision.}
  \label{fig:case_study}
\end{figure*}%
 All reported results are means over three independent runs with different random seeds. Evaluation settings, including decoding parameters and answer verification, are described in Appendix~\ref{sec:eval_details}.

\subsection{Main Results}
\label{sec:main_results}

We evaluate \method across six mathematical reasoning and six code-generation benchmarks, with three student--teacher pairs in each domain. Tables~\ref{tab:main} and~\ref{tab:code_main} report the main results.
 Across all six settings, \method achieves the highest average performance among the compared distillation methods, demonstrating the effectiveness of optimizing synthetic questions for the target student--teacher pair.

\paragraph{Mathematical reasoning.}
As shown in Table~\ref{tab:main}, \method achieves average scores of 46.2, 51.6, and 64.5 for Qwen3-1.7B, R1-Distill-1.5B, and R1-Distill-7B, respectively, compared with 43.0, 50.5, and 62.9 for Vanilla OPD. These results show that optimized synthetic questions can outperform the original real question bank across both non-thinking and reasoning students. The corresponding teacher-gap recovery rates reach 90.2\%, 99.9\%, and 81.3\%; notably, R1-Distill-1.5B nearly matches its teacher's average performance. Relative to Student-Generated OPD, our method improves the reported average scores by 5.1, 3.4, and 3.7 percentage points, indicating that optimizing the synthesizer provides additional value beyond sampling questions directly from the student.

\method also outperforms KL-Filtered OPD, Evol-Instruct, and R-Zero on the average score in every mathematical setting. The comparison with KL-Filtered OPD supports the value of optimizing the prompt distribution using support-aware supervision beyond selecting real questions with large teacher--student disagreement. At the benchmark level, representative gains over Vanilla OPD include an increase from 20.0 to 26.5 on AIME26 for Qwen3-1.7B and from 78.1 to 83.6 on AMC23 for R1-Distill-1.5B. The improvements are not uniform across benchmarks: other methods remain stronger on some BRUMO25 and MATH500 evaluations. Nevertheless, the consistent average gains indicate that our optimized question banks provide effective supervision across the evaluated mathematical tasks.

\paragraph{Code generation.}
Table~\ref{tab:code_main} shows that \method also improves code-generation performance across all three student--teacher pairs.
 Its average scores reach 39.6, 35.7, and 21.2, compared with 34.5, 30.2, and 20.2 for Vanilla OPD, respectively. Our method also exceeds Student-Generated OPD, KL-Filtered OPD, Evol-Instruct, and SQLM in average performance, and achieves the best or tied-best score among the compared distillation methods in 16 of the 18 benchmark--pair combinations. The gains extend across multiple coding benchmarks. For example, with Qwen3-30B-A3B as the teacher, APPS-Interview improves from 21.3 to 31.6 and CodeContests from 4.8 to 13.9 relative to Vanilla OPD. These results extend the benefits of optimizing elicited supervision beyond mathematical reasoning to code generation.

The magnitude of improvement nevertheless varies substantially across student--teacher pairs. Teacher-gap recovery reaches 82.9\% for Qwen3-4B distilled from the 4B teacher, but remains at 32.8\% with the stronger 30B-A3B teacher and 12.3\% for the Qwen2.5-Coder-3B/32B pair. Moreover, for the same Qwen3-4B student, the stronger teacher yields a lower distilled average score than the 4B teacher under our setup. This comparison indicates that teacher benchmark performance alone does not determine downstream distillation effectiveness. Optimizing the question distribution consistently improves the evaluated outcomes, while substantial headroom remains in some student--teacher settings.
\subsection{Ablation: Can a stronger model replace synthetic data optimization?}
\label{sec:synthesizer_source}

A natural alternative to training a synthesizer for each student--teacher pair is
to ask a stronger off-the-shelf model---in our setting, the teacher itself---to
generate the questions. To separate the effect of synthesizer initialization from
that of pair-specific optimization, we compare student- and teacher-initialized
synthesizers, each with and without RL, while holding the downstream distillation
setup fixed. Figure~\ref{fig:synthesizer_source} reports performance over training
for a representative code student--teacher pair; the dashed lines mark the
student and teacher reference scores.

Without RL, teacher initialization improves over student initialization on the
two APPS subsets, but not on LCB, and neither unoptimized synthesizer matches the
student-initialized RL variant across the three benchmarks. Applying RL improves
the trajectories for both initializations, indicating that the benefit is not
explained by simply prompting a stronger model to write questions. The
student-initialized RL variant reaches roughly 35\% on LCB, above the teacher
reference of 33.7\%, and about 62\% on APPS-Intro. On APPS-Interview it reaches
about 36\%; both APPS scores remain below their teacher references (65.6\% and
41.0\%, respectively). The teacher-initialized RL variant is competitive---and
slightly higher at some APPS-Intro checkpoints---but does not consistently
dominate the student-initialized variant. Thus, synthesizer strength alone is not
a substitute for optimizing questions against the target pair, while the choice
of initialization can still affect performance by benchmark. These curves are
descriptive: they show checkpoint trends, not uncertainty across independent
runs.

\subsection{Ablation: which reward components matter?}
\label{sec:reward_ablation}

We ablate compatibility, the training-time diversity term, and the explicit
length penalty, and evaluate the resulting question banks by downstream best
accuracy on AIME24, AIME25, and AIME26 (Figure~\ref{fig:reward_ablation}). The
full reward performs best on all three benchmarks, reaching 45.0\%, 35.8\%, and
29.2\%, respectively. Removing any one of the three components lowers accuracy
on every benchmark, with drops ranging from 2.1 to 6.6 percentage points.

The size of each drop depends on the benchmark. On AIME25, removing
compatibility, diversity, or the length penalty reduces accuracy by 6.2, 6.0,
and 6.6 points, respectively. On AIME24, the largest reduction is 4.2 points
without diversity; on AIME26, it is 3.4 points without the length penalty. These
results suggest that all three components contribute to the performance of the
complete reward, while their relative importance varies across benchmarks. The
figure reports downstream accuracy rather than synthesizer reward, directly
assessing the quality of the prompts for distillation.

\subsection{Applicability to agentic tasks}
\label{sec:agentic_tasks}

To test whether \method applies beyond mathematical reasoning and code
generation, we evaluate it on the three agentic task domains in $\tau^2$-Bench:
Airline, Retail, and Telecom. Table~\ref{tab:tau2_results} reports results for
Qwen3-4B (non-thinking).

Our method reaches 46.3\%, 49.4\%, and 30.0\% in the three domains,
respectively, for an average of 41.9\%. This exceeds Student-Generated OPD by
3.8, 6.9, and 3.1 percentage points, respectively (4.6 points on average), and
raises teacher-gap recovery from 79.2\% to 95.2\%. It also outperforms Vanilla
OPD (40.4\% average) and KL-Filtered OPD (40.9\%) on the average score. For this
student--teacher pair, these results support the applicability of pair-specific
question optimization to agentic task domains as well as to reasoning and code
generation. The average remains 1.4 percentage points below the teacher
reference (43.3\%).

\subsection{Cross-Domain Transfer between Math and Code}
\label{sec:cross_domain_transfer}

Our distribution-elicitation perspective suggests that the benefits of OPD need not be confined to the domain of its training questions. We examine this hypothesis in two directions: code-to-math transfer with Qwen3-4B (non-thinking) distilled from Qwen3-4B-Instruct-2507, and math-to-code transfer with R1-Distill-7B distilled from Skywork-OR1-Math-7B. In each setting, we directly evaluate the student after source-domain OPD, without additional target-domain training, comparing our optimized synthetic questions with the real source-domain questions used by Vanilla OPD. We report gains over the initial student and compute an unweighted average across the target benchmarks.

\paragraph{Code-to-math transfer.}
We evaluate code-trained students on AIME24, AIME26, AMC23, and BRUMO25. As shown in Figure~\ref{fig:math_to_code}(a), both methods improve performance on all four mathematics benchmarks. \method raises average accuracy from 31.38\% to 38.60\%, a gain of 7.23 percentage points, compared with 5.33 points for Vanilla OPD. Our method consistently yields larger improvements across the four benchmarks, with gains of 11.0, 5.2, 8.5, and 4.2 points, respectively. These results show that supervision elicited by code questions can also benefit mathematical reasoning.

\paragraph{Math-to-code transfer.}
We evaluate mathematics-trained students on LiveCodeBench, APPS-Intro, and APPS-Interview. Figure~\ref{fig:math_to_code}(b) shows that \method improves performance by 2.7, 4.8, and 8.7 percentage points, respectively, compared with 1.9, 3.8, and 6.7 points for Vanilla OPD. Average accuracy increases from 47.00\% to 52.40\% with our method, exceeding Vanilla OPD's 51.13\% and approaching the teacher's 52.70\%. The student also surpasses the teacher on APPS-Intro (70.1\% versus 68.7\%), while remaining below it on the other two benchmarks.

Together, these results demonstrate cross-domain gains in both evaluated directions, with our optimized synthetic questions outperforming real source-domain questions on every reported target benchmark. This finding is consistent with questions serving as carriers of distributional supervision whose benefits extend beyond the training domain. Although these experiments do not isolate the mechanism underlying transfer, they show that optimizing the supervision elicited from a particular student--teacher pair can improve performance beyond source-domain tasks.

\begin{figure*}[t]
  \centering
  \includegraphics[width=\textwidth]{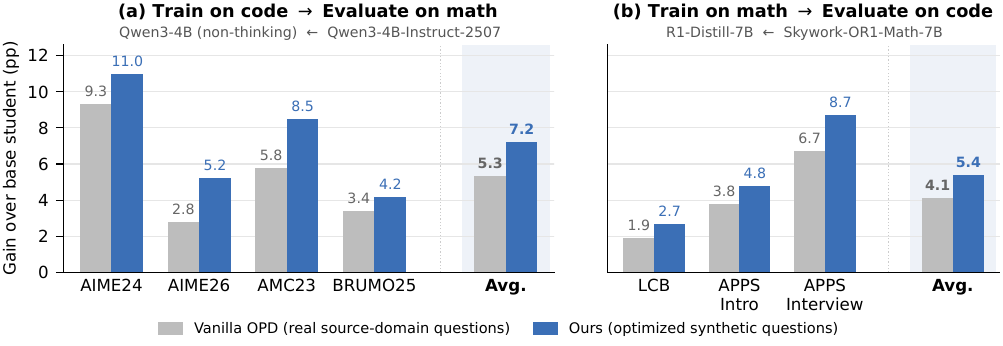}
  \caption{\textbf{Cross-domain transfer between math and code.} Gains over the initial student after source-domain OPD, without additional target-domain training. (a) Code-to-math transfer with Qwen3-4B (non-thinking) distilled from Qwen3-4B-Instruct-2507. (b) Math-to-code transfer with R1-Distill-7B distilled from Skywork-OR1-Math-7B. \method uses optimized synthetic source-domain questions, while Vanilla OPD uses real source-domain questions. Bars report gains in percentage points; shaded groups show unweighted averages across the displayed benchmarks. Both methods improve all seven evaluated pair--benchmark combinations, and \method consistently yields larger gains.}
  \label{fig:math_to_code}
\end{figure*}

\subsection{Training-token efficiency and checkpoint gains}
\label{sec:bank_efficiency}

We compare downstream accuracy against cumulative training tokens for
R1-Distill-1.5B on JustRL and R1-Distill-7B on Skywork
(Figure~\ref{fig:opd_token_efficiency}). The plot shows the Vanilla and \method
checkpoint trajectories, the Vanilla best-so-far reference, and the best
checkpoint for each method. \method reaches a best accuracy of 54.80\% on
R1-Distill-1.5B, compared with 52.27\% for Vanilla, a gain of 2.53 percentage
points. On R1-Distill-7B, \method reaches 63.80\% versus 61.70\% for Vanilla, a
gain of 2.10 points. The best-checkpoint token counts are similar for the
1.5B model, while \method reaches its best checkpoint with substantially fewer
training tokens for the 7B model.

These trajectories indicate that the accuracy gains are achieved without a
substantial increase in downstream training-token use, particularly for
R1-Distill-7B. They report downstream accuracy and token use only; they do not
establish lower end-to-end cost, which also depends on the upstream cost of
generating and scoring synthetic questions.

\subsection{Iterative Training}
\label{sec:iterative_training}

Our main framework optimizes synthetic questions for a fixed student--teacher pair before performing OPD. We further examine whether repeating this process as the student improves yields additional benefits. We conduct three rounds on the code-generation pair consisting of a Qwen3-4B student and a Qwen3-30B-A3B teacher. After each round, we update the student used for question optimization, construct a new synthetic question bank, and continue OPD while keeping the teacher fixed. We evaluate on LiveCodeBench, APPS-Intro, and APPS-Interview and report the highest observed average accuracy across the evaluated checkpoints within each round. This comparison provides an optimistic assessment of the benefits of additional rounds, rather than a validation-selected final result.

The first round raises average accuracy from 29.03\% to 40.53\%, an improvement of 11.50 percentage points. In contrast, the best observed averages in the second and third rounds reach only 40.82\% and 40.97\%, respectively. Thus, two additional rounds yield a cumulative gain of just 0.44 points over the first round, despite the additional question-generation and distillation effort. Even under this best-checkpoint comparison, iterative training provides only marginal improvements in the evaluated setting.

One possible explanation is that the first round already captures much of the readily accessible supervision for this student--teacher pair. As the student becomes better aligned with the teacher, the synthesizer may find it increasingly difficult to identify questions that expose substantial disagreement while remaining within regions where the teacher provides useful guidance. In terms of our learning-value objective, maintaining both disagreement and compatibility may therefore become harder across rounds. This interpretation remains a hypothesis, as downstream scores alone do not establish how these signals evolve during training. Nevertheless, the observed diminishing returns suggest that repeatedly adapting the question distribution is not sufficient to sustain the large gains achieved in the first round, supporting single-round optimization as a practical choice in this setting.

\subsection{Case study: where does the teacher provide useful signal?}
\label{sec:case_study}

To examine what the learning-value reward captures at the token level,
Figure~\ref{fig:case_study} contrasts a real geometry question from DAPO-17k-Math
with a short combinatorics question posed by our student-initialized
synthesizer, both scored with the same student--teacher pair. Each panel reports
the apparent difficulty of the question, the mean teacher--student discrepancy,
and the mean support-aware score $s_t$ (\S\ref{sec:reward}) over one student
rollout. Below these statistics, an excerpt of the rollout is shaded by $s_t$
and discrepancy, and two strips summarize both quantities over the full
response, with a box marking the excerpt. The student answers both questions
incorrectly.

\paragraph{Signal density.}
Despite its simpler surface form, the synthetic question elicits denser
supervision. Its rollout contains 1,241 tokens, compared with 2,002 for the
real question, while achieving approximately $1.5\times$ the mean $s_t$
(0.172 versus 0.117) and $1.7\times$ the mean discrepancy (0.513 versus 0.307).
Thus, the synthetic question produces a 38\% shorter trajectory with a higher
average support-aware score. Its cumulative score remains comparable to that of the real question
($\sum_t s_t = 214$ versus $234$), retaining approximately 91\% of the scored
signal with only 62\% of the response tokens.

\paragraph{Where the signal arises.}
The excerpts indicate where this difference comes from. In the geometry
rollout, the displayed segment consists of routine coordinate algebra: the
teacher largely agrees with the student token by token, and most positions
receive low scores. In the combinatorics rollout, high-scoring positions
cluster at the decisions that structure the solution---treating B and S as a
single entity and then counting the resulting items---including the step
where the student counts three items although there are four (U, E, the
B--S block, and C). In both rollouts, $s_t$ is almost perfectly
rank-correlated with the discrepancy (Spearman $\rho=0.994$ and $0.997$).
Thus, the combined score largely preserves the ordering of positions by
discrepancy in these examples. This rank correlation alone does not determine
the magnitude of the compatibility factor.

\begin{figure}[!t]
  \centering
  \includegraphics[width=0.45\textwidth]{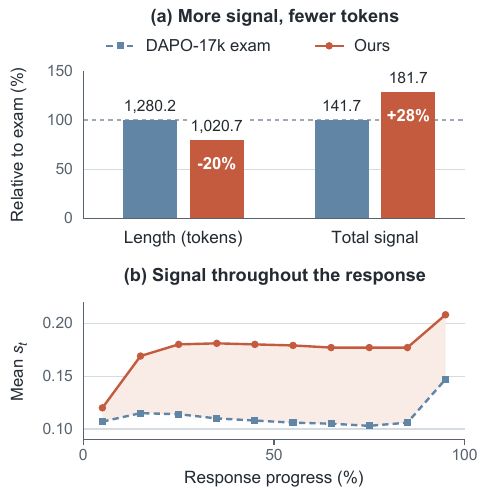}
  \caption{\textbf{Optimized questions elicit more signal with fewer tokens.} We compare student rollouts on 1,000 questions each from DAPO-17k-Math and our synthetic bank. (a) Mean response length and cumulative support-aware score $\sum_t s_t$, normalized to the real-question baseline (100\%); labels above bars show raw means. Synthetic questions yield 20\% shorter responses and 28\% more cumulative signal. (b) Mean token-level score $s_t$ across relative response-position bins. Synthetic questions maintain higher scores throughout the response, indicating that the signal-density advantage extends across the trajectory.}
  \label{fig:signal_efficiency}
\end{figure}

This example illustrates that surface complexity alone is an incomplete
indicator of a question's value for OPD: a concise, apparently straightforward
problem can elicit dense teacher--student disagreement with substantial
support-aware signal. The observed advantage concerns signal density; the
cumulative support-aware score is slightly lower for the synthetic question.
These observations provide a qualitative illustration of our
distribution-elicitation perspective, complementing the downstream performance
comparisons.

\paragraph{Aggregate signal analysis.} To complement the individual examples, we sample 1,000 questions each from DAPO-17k-Math and our synthetic question bank and analyze the support-aware signal along student rollouts. As shown in Figure~\ref{fig:signal_efficiency}(a), synthetic questions elicit shorter responses on average (1,020.7 versus 1,280.2 tokens, a 20\% reduction), yet yield a higher mean cumulative score $\sum_t s_t$ (181.7 versus 141.7, a 28\% increase). Thus, across these samples, the higher signal density translates into more total scored signal despite fewer response tokens. Figure~\ref{fig:signal_efficiency}(b) further shows that the mean $s_t$ remains higher for synthetic questions across all displayed relative-position bins, indicating that the advantage extends throughout the response rather than being confined to a narrow segment. Together with the concrete examples, these aggregate results support our distribution-elicitation perspective: optimized questions can expose denser and more abundant support-aware supervision without relying on longer student trajectories.

\subsection{Is Elicited Divergence a Learning Signal?}
Our reward rests on the premise that questions eliciting larger teacher--student divergence along student trajectories provide more useful supervision. On natural questions, the support-aware score closely tracks its divergence component (Spearman $\rho(\bar{s},\bar{\delta})=0.90$). We examine whether these signals predict actual learning gains through two experiments on the Qwen3-1.7B $\leftarrow$ Qwen3-4B-RL-Math pair.

\paragraph{Per-question learning value.}
For each of 1,500 held-out DAPO questions, we apply a single OPD update starting from the same base student and measure the resulting decrease in teacher--student KL on a fixed probe set. As shown in Figure~\ref{fig:divergence_learning_value} (left), mean divergence $\bar{\delta}$ strongly correlates with this learning value ($\rho=0.727$, 95\% CI $[0.695,0.755]$). The association remains positive after controlling for response length, student pass rate, and student entropy (partial $\rho=0.386$, 95\% CI $[0.329,0.443]$). The full score $\bar{s}$ exhibits similar associations, with raw and partial correlations of $0.750$ and $0.399$, respectively. Consistently, probe KL decrease rises monotonically across quintiles ranked by $\bar{s}$, with the highest quintile yielding $2.15\times$ the decrease of the lowest (middle).

\paragraph{Downstream training utility.}
We further train students on five score-ranked question banks under an identical token budget of 24.4M tokens, using three seeds per bank. Figure~\ref{fig:divergence_learning_value} (right) shows that higher-scoring banks generally yield better final accuracy across six benchmarks. Average accuracy rises from $41.06\%$ in R1 to $43.49\%$ in R4, while R5 reaches $43.11\%$, exceeding R1 by $2.05$ percentage points. Thus, the overall trend favors higher-scoring banks, although performance is not strictly monotonic. Additional high-versus-low divergence comparisons retain a positive advantage after balancing compatibility (Appendix~\ref{app:divergence_controls}). Together, these results support elicited divergence as a useful indicator of distillation value, with evidence from both single-update gains and downstream training.

Both experiments use natural DAPO-17k-Math questions, without synthesizer outputs, to evaluate whether the proposed signals predict learning value independently of synthesizer optimization. The one-step analysis measures changes on a fixed, disjoint probe set, while downstream comparisons use a common training-token budget. Appendix~\ref{app:signal_validation} provides the data splits, update and evaluation protocols, controlled designs, and statistical procedures.

\begin{figure*}[t]
    \centering
    \includegraphics[width=0.32\textwidth]{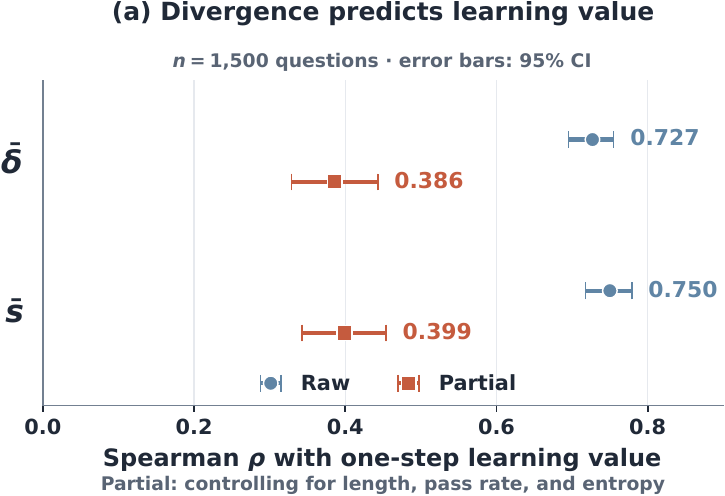}\hfill
    \includegraphics[width=0.32\textwidth]{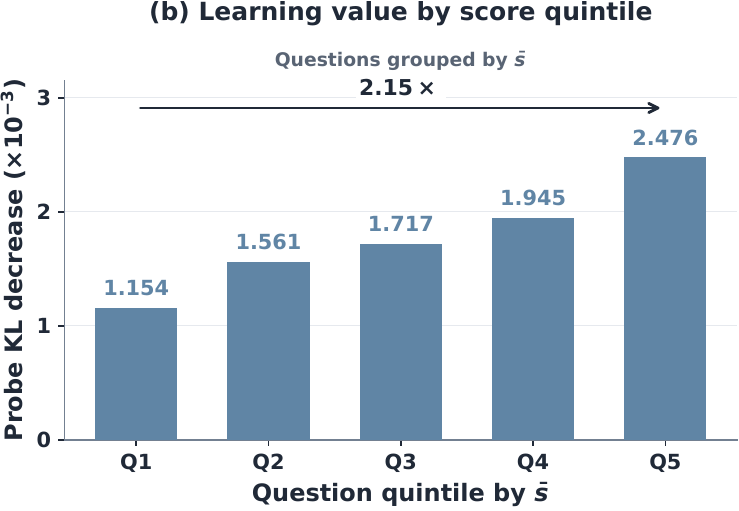}\hfill
    \includegraphics[width=0.32\textwidth]{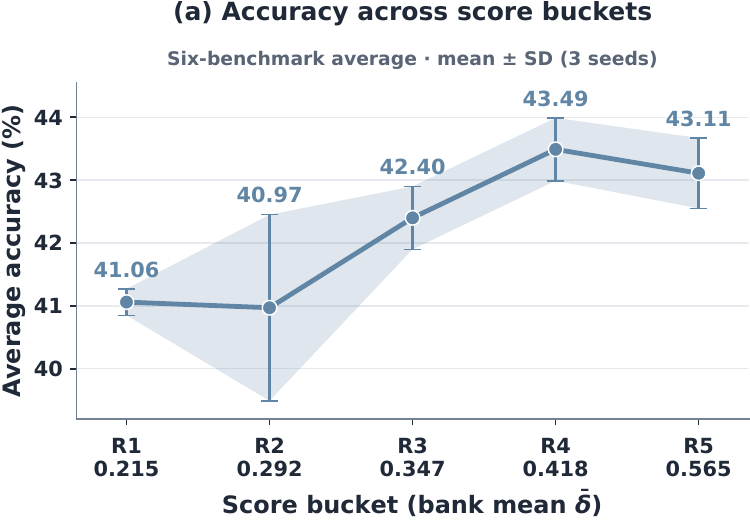}
    \caption{
    \textbf{Teacher--student divergence is positively associated with learning value.}
    \textbf{Left:} Raw and partial Spearman correlations of mean divergence
    $\bar{\delta}$ and mean score $\bar{s}$ with one-step learning value
    over 1,500 questions. Partial correlations control for response length,
    pass rate, and entropy; error bars indicate 95\% confidence intervals.
    \textbf{Middle:} Probe KL decrease increases across question quintiles
    ranked by $\bar{s}$, with the highest quintile yielding
    $2.15\times$ the decrease of the lowest.
    \textbf{Right:} Final average accuracy over six benchmarks for training
    on five score buckets, annotated with their mean $\bar{\delta}$.
    Points show means over three seeds; error bars and shading indicate
    $\pm$1 standard deviation.
    }
    \label{fig:divergence_learning_value}
\end{figure*}

\section{Related Work}
\label{sec:related_work}

\subsection{On-Policy Distillation and Selective Supervision}

On-policy distillation supervises student-generated trajectories using teacher predictions at the same visited states. GKD addresses the distribution mismatch of fixed-response distillation by training on student-generated sequences~\citep{gkd}, while MiniLLM develops an on-policy optimization approach based on reverse KL divergence~\citep{minillm}. Subsequent work improves the trajectories and teacher targets used for distillation. Speculative Knowledge Distillation interleaves student and teacher sampling to improve training sequences while accounting for the student's generation distribution~\citep{speculative-opd}. Best-of-N Teacher Rollout Selection augments student-context supervision with an auxiliary teacher trajectory selected according to correctness and student alignment~\citep{best-n-opd}. PromptKD instead tunes teacher-side prompts to produce student-friendly supervision~\citep{promptKD}. These approaches modify how supervision is produced or delivered, whereas our method learns the questions that elicit it.

Selective OPD further recognizes that response positions provide unequal learning value. TIP characterizes token importance through student entropy and teacher--student divergence~\citep{tip_opd}. Entropy-Aware OPD augments reverse KL with forward KL at high-teacher-entropy positions~\citep{entropy}, and FiRe-OPD combines trajectory filtering with soft token reweighting~\citep{filter-opd}. Token Teachability distinguishes locally learnable disagreement from mismatch outside the student's candidate support, using teacher--student discrepancy and teacher probability mass on the student's top-$K$ candidates~\citep{opd_token_teach}. In the related setting of on-policy self-distillation, matching supervision to student capacity has also been studied through joint token weighting and privileged-information adjustment~\citep{yang2026matchingsupervisionstudentslearning}. We build on the support-aware characterization of teachability, adapting it into a question-level reward with trajectory aggregation, length control, and batch-diversity control. Our optimization changes the synthetic prompt distribution while holding the downstream OPD procedure fixed.

\subsection{Synthetic Data Generation and Selection for Distillation}

Synthetic-data methods differ in the feedback used to construct training examples. Self-Instruct generates instruction--response pairs for supervised finetuning~\citep{self-instruct}, and WizardLM evolves instructions toward greater complexity~\citep{xu2024wizardlm}. R-Zero and Self-Questioning Language Models train task proposers through solver feedback to support self-improving reasoning~\citep{huang2025rzero,chen2025selfquestioning}. Self-Adapting Language Models learns to generate finetuning data and update directives using the adapted model's downstream performance~\citep{self-adaptive}. Particularly relevant to our setting, Prompting to Distill uses reinforced prompts and teacher--student discrepancy to synthesize data for text-classification distillation~\citep{prompt_to_distill}. Our work extends the broader idea of feedback-guided synthesis to prompt-only OPD data, where value is assessed through teacher distributions along student-generated reasoning trajectories.

Recent studies directly examine the role of training questions in OPD. Rethinking OPD identifies teacher--student compatibility as an important factor and proposes teacher-aligned prompt selection~\citep{rethink_opd1}. Its follow-up shows that a single query can expose broad state coverage and recover much of the gain from full-data training~\citep{rethink_opd2}. What Matters in OPD studies difficulty-based selection, finding that a small set of hard questions can match a much larger dataset in its evaluated settings and relating their utility to longer reasoning trajectories~\citep{opd-data-efficiency}. Beyond Prompt Count shows that prompt usefulness depends on the student--teacher pair and target capability, with targeted selection not consistently outperforming random sampling~\citep{wang2026promptcountdatashapes}. Data-Free OPD further demonstrates effective distillation using self-generated questions without external seed examples~\citep{data-free}.

\method learns a synthesizer to actively shape the question distribution for a target student--teacher pair. Its objective evaluates the support-aware supervision elicited by each question, rather than using surface complexity, answer correctness, or raw disagreement alone as the generation criterion. Unlike selection from an existing question bank or direct sampling from an unoptimized model, our pipeline trains the synthesizer using a frozen scoring pair and then constructs a curated prompt-only bank for OPD. Question generation and synthesizer training reuse the target student and teacher without external training questions or auxiliary models. A frozen embedding model is used only to remove semantically redundant questions during post-generation curation.

\section{Conclusion}

We introduced \method, a framework that treats synthetic questions as distribution elicitors for on-policy distillation. Rather than judging a question by its surface form, \method trains a student-initialized synthesizer against a frozen student--teacher pair to expose useful teacher supervision along student-generated trajectories. Its learning-value reward combines teacher--student disagreement with teacher probability mass on the student's candidate support, together with length and diversity controls, and the generated questions are filtered and deduplicated into a prompt-only bank for standard OPD. Synthesizer training reuses the target student and teacher without external training questions or auxiliary models; a frozen embedding model is used only for post-generation semantic deduplication. Across mathematical reasoning, code generation, and agentic tasks, \method improves average performance and teacher-gap recovery over unoptimized student-generated prompts and KL-filtered real questions, while matching or exceeding real-bank OPD in the evaluated settings. Signal analyses further show that optimized questions can elicit denser supervision from shorter student responses. These results support elicited distributional supervision as an objective for constructing synthetic distillation data.

Our learning-value reward is a surrogate estimated from a single student response per question, and the synthesizer is optimized once for a fixed student--teacher pair. Developing more effective adaptive question-generation strategies as the student improves, and combining learned question generation with token-level selection methods for OPD, are natural directions for future work.

\bibliography{example_paper}

@misc{yang2026matchingsupervisionstudentslearning,
      title={Matching Supervision to the Student's Learning Capacity: A Unified Framework for On-Policy Self-Distillation}, 
      author={Yongkang Yang and Zhezheng Hao and Hong Zhang and Yi Liu and Xiankun Lin and Wence Ji and Fanjunduo Wei and Jiarui Yu and Qiang Lin and Xiaoyun Liang and Hande Dong},
      year={2026},
      eprint={2608.08176},
      archivePrefix={arXiv},
      primaryClass={cs.AI},
      url={https://arxiv.org/abs/2608.08176}, 
}

@inproceedings{xu2024wizardlm,
  author       = {Can Xu and
                  Qingfeng Sun and
                  Kai Zheng and
                  Xiubo Geng and
                  Pu Zhao and
                  Jiazhan Feng and
                  Chongyang Tao and
                  Qingwei Lin and
                  Daxin Jiang},
  title        = {WizardLM: Empowering Large Pre-Trained Language Models to Follow Complex
                  Instructions},
  booktitle    = {The Twelfth International Conference on Learning Representations,
                  {ICLR} 2024, Vienna, Austria, May 7-11, 2024},
  publisher    = {OpenReview.net},
  year         = {2024},
  url          = {https://openreview.net/forum?id=CfXh93NDgH},
  bibsource    = {dblp computer science bibliography, https://dblp.org}
}

@inproceedings{huang2025rzero,
  author       = {Chengsong Huang and
                  Wenhao Yu and
                  Xiaoyang Wang and
                  Hongming Zhang and
                  Zongxia Li and
                  Ruosen Li and
                  Jiaxin Huang and
                  Haitao Mi and
                  Dong Yu},
  editor       = {Carl Vondrick and
                  Bharath Hariharan and
                  Colin Raffel and
                  Lerrel Pinto and
                  Diyi Yang and
                  Aleksandra Faust},
  title        = {R-Zero: Self-Evolving Reasoning {LLM} from Zero Data},
  booktitle    = {The Fourteenth International Conference on Learning Representations,
                  {ICLR} 2026, Rio de Janeiro, Brazil, April 23-27, 2026},
  publisher    = {proceedings.iclr.cc},
  year         = {2026},
  url          = {https://proceedings.iclr.cc/paper\_files/paper/2026/hash/d49b9aacebda61051166335af6fd3061-Abstract-Conference.html},
  bibsource    = {dblp computer science bibliography, https://dblp.org}
}

@article{chen2025selfquestioning,
  author       = {Lili Chen and
                  Mihir Prabhudesai and
                  Katerina Fragkiadaki and
                  Hao Liu and
                  Deepak Pathak},
  title        = {Self-Questioning Language Models},
  journal      = {CoRR},
  volume       = {abs/2508.03682},
  year         = {2025},
  url          = {https://doi.org/10.48550/arXiv.2508.03682},
  doi          = {10.48550/ARXIV.2508.03682},
  eprinttype   = {arXiv},
  eprint       = {2508.03682},
  bibsource    = {dblp computer science bibliography, https://dblp.org}
}

@article{rethink_opd1,
  author       = {Yaxuan Li and
                  Yuxin Zuo and
                  Bingxiang He and
                  Jinqian Zhang and
                  Chaojun Xiao and
                  Cheng Qian and
                  Tianyu Yu and
                  Huan{-}ang Gao and
                  Wenkai Yang and
                  Zhiyuan Liu and
                  Ning Ding},
  title        = {Rethinking On-Policy Distillation of Large Language Models: Phenomenology,
                  Mechanism, and Recipe},
  journal      = {CoRR},
  volume       = {abs/2604.13016},
  year         = {2026},
  url          = {https://doi.org/10.48550/arXiv.2604.13016},
  doi          = {10.48550/ARXIV.2604.13016},
  eprinttype   = {arXiv},
  eprint       = {2604.13016},
  bibsource    = {dblp computer science bibliography, https://dblp.org}
}

@article{rethink_opd2,
  author       = {Zixuan Fu and
                  Bingxiang He and
                  Yuxin Zuo and
                  Haohuan Huang and
                  Jinqian Zhang and
                  Ruhang Xiao and
                  Cheng Qian and
                  Qinyu Luo and
                  Huan{-}ang Gao and
                  Yudong Wang and
                  Zhiyuan Liu and
                  Ning Ding and
                  Chaojun Xiao},
  title        = {Rethinking On-Policy Distillation of Large Language Models {II:} One
                  Training Example},
  journal      = {CoRR},
  volume       = {abs/2609.04172},
  year         = {2026},
  url          = {https://doi.org/10.48550/arXiv.2609.04172},
  doi          = {10.48550/ARXIV.2609.04172},
  eprinttype   = {arXiv},
  eprint       = {2609.04172},
  bibsource    = {dblp computer science bibliography, https://dblp.org}
}

@article{data-free,
  author       = {Gengsheng Li and
                  Mao Zheng and
                  Mingyang Song and
                  Jie Sun and
                  Zeyuan Liu and
                  Ruiqi Liu and
                  Qiyong Zhong and
                  Haiyun Guo and
                  Junfeng Fang and
                  Jinqiao Wang},
  title        = {Data-free On-policy Distillation},
  journal      = {CoRR},
  volume       = {abs/2609.14193},
  year         = {2026},
  url          = {https://doi.org/10.48550/arXiv.2609.14193},
  doi          = {10.48550/ARXIV.2609.14193},
  eprinttype   = {arXiv},
  eprint       = {2609.14193},
  bibsource    = {dblp computer science bibliography, https://dblp.org}
}

@article{best-n-opd,
  author       = {Ke Zhang and
                  Yunjie Tian and
                  Dongdi Zhao and
                  Yijiang Li and
                  Yuanye Liu and
                  Vishal M. Patel and
                  Di Fu},
  title        = {On-Policy Distillation with Best-of-N Teacher Rollout Selection},
  journal      = {CoRR},
  volume       = {abs/2605.09725},
  year         = {2026},
  url          = {https://doi.org/10.48550/arXiv.2605.09725},
  doi          = {10.48550/ARXIV.2605.09725},
  eprinttype   = {arXiv},
  eprint       = {2605.09725},
  bibsource    = {dblp computer science bibliography, https://dblp.org}
}

@article{opd-data-efficiency,
  author       = {Zhinan Hou and
                  Jiaqi Zhang and
                  Xunliang Cai and
                  Keyou You},
  title        = {What Matters in On-Policy Distillation? {A} Perspective on Data Efficiency
                  and Data Selection},
  journal      = {CoRR},
  volume       = {abs/2609.05198},
  year         = {2026},
  url          = {https://doi.org/10.48550/arXiv.2609.05198},
  doi          = {10.48550/ARXIV.2609.05198},
  eprinttype   = {arXiv},
  eprint       = {2609.05198},
  bibsource    = {dblp computer science bibliography, https://dblp.org}
}

@misc{wang2026promptcountdatashapes,
      title={Beyond Prompt Count: How Data Shapes Transfer in On-Policy Distillation}, 
      author={Jiaxuan Wang and Jiafei Lyu and Yuchen Cai and Siye Wu and Pengyuan Wang and Jiashun Liu and Xiang Cheng and Kai Yang and Yangkun Chen and Saiyong Yang and Lan-Zhe Guo},
      year={2026},
      eprint={2609.37377},
      archivePrefix={arXiv},
      primaryClass={cs.AI},
      url={https://arxiv.org/abs/2609.37377}, 
}

@article{opd_token_teach,
  author       = {Yuanyi Wang and
                  Su Lu and
                  Yanggan Gu and
                  Pengkai Wang and
                  Yifan Yang and
                  Zhaoyi Yan and
                  Congkai Xie and
                  Jianmin Wu and
                  Hongxia Yang},
  title        = {Not All Disagreement Is Learnable: Token Teachability in On-Policy
                  Distillation},
  journal      = {CoRR},
  volume       = {abs/2605.26844},
  year         = {2026},
  url          = {https://doi.org/10.48550/arXiv.2605.26844},
  doi          = {10.48550/ARXIV.2605.26844},
  eprinttype   = {arXiv},
  eprint       = {2605.26844},
  bibsource    = {dblp computer science bibliography, https://dblp.org}
}

@inproceedings{dapo,
  author       = {Qiying Yu and
                  Zheng Zhang and
                  Ruofei Zhu and
                  Yufeng Yuan and
                  Xiaochen Zuo and
                  Yu Yue and
                  Weinan Dai and
                  Tiantian Fan and
                  Gaohong Liu and
                  Juncai Liu and
                  Lingjun Liu and
                  Xin Liu and
                  Haibin Lin and
                  Zhiqi Lin and
                  Bole Ma and
                  Guangming Sheng and
                  Yuxuan Tong and
                  Chi Zhang and
                  Mofan Zhang and
                  Ru Zhang and
                  Wang Zhang and
                  Hang Zhu and
                  Jinhua Zhu and
                  Jiaze Chen and
                  Jiangjie Chen and
                  Chengyi Wang and
                  Hongli Yu and
                  Yuxuan Song and
                  Xiangpeng Wei and
                  Hao Zhou and
                  Jingjing Liu and
                  Wei{-}Ying Ma and
                  Ya{-}Qin Zhang and
                  Lin Yan and
                  Yonghui Wu and
                  Mingxuan Wang},
  editor       = {Danielle Belgrave and
                  Cheng Zhang and
                  Laura N. Montoya and
                  Hsuan{-}Tien Lin and
                  Razvan Pascanu and
                  Piotr Koniusz and
                  Marzyeh Ghassemi and
                  Nancy Chen and
                  Iv{\'{a}}n Vladimir Meza Ru{\'{\i}}z and
                  Arturo Loaiza{-}Bonilla},
  title        = {{DAPO:} An Open-Source {LLM} Reinforcement Learning System at Scale},
  booktitle    = {Advances in Neural Information Processing Systems 38: Annual Conference
                  on Neural Information Processing Systems 2025, NeurIPS 2025, San Diego,
                  CA, USA, December 2-7, 2025 / Mexico City, Mexico, November 30 - December
                  5, 2025},
  year         = {2025},
  url          = {http://papers.nips.cc/paper\_files/paper/2025/hash/a4277440d50f1f15d2cb4c14f7e0c0d2-Abstract-Conference.html},
  bibsource    = {dblp computer science bibliography, https://dblp.org}
}

@misc{tip_opd,
      title={TIP: Token Importance in On-Policy Distillation}, 
      author={Yuanda Xu and Hejian Sang and Zhengze Zhou and Ran He and Zhipeng Wang and Alborz Geramifard},
      year={2026},
      eprint={2604.14084},
      archivePrefix={arXiv},
      primaryClass={cs.LG},
      url={https://arxiv.org/abs/2604.14084}, 
}

@article{code-dataset,
  author       = {Wasi Uddin Ahmad and
                  Sean Narenthiran and
                  Somshubra Majumdar and
                  Aleksander Ficek and
                  Siddhartha Jain and
                  Jocelyn Huang and
                  Vahid Noroozi and
                  Boris Ginsburg},
  title        = {OpenCodeReasoning: Advancing Data Distillation for Competitive Coding},
  journal      = {CoRR},
  volume       = {abs/2504.01943},
  year         = {2025},
  url          = {https://doi.org/10.48550/arXiv.2504.01943},
  doi          = {10.48550/ARXIV.2504.01943},
  eprinttype   = {arXiv},
  eprint       = {2504.01943},
  bibsource    = {dblp computer science bibliography, https://dblp.org}
}

@inproceedings{prompt_to_distill,
  author       = {Xinyin Ma and
                  Xinchao Wang and
                  Gongfan Fang and
                  Yongliang Shen and
                  Weiming Lu},
  editor       = {Luc De Raedt},
  title        = {Prompting to Distill: Boosting Data-Free Knowledge Distillation via
                  Reinforced Prompt},
  booktitle    = {Proceedings of the Thirty-First International Joint Conference on
                  Artificial Intelligence, {IJCAI} 2022, Vienna, Austria, 23-29 July
                  2022},
  pages        = {4296--4302},
  publisher    = {ijcai.org},
  year         = {2022},
  url          = {https://doi.org/10.24963/ijcai.2022/596},
  doi          = {10.24963/IJCAI.2022/596},
  bibsource    = {dblp computer science bibliography, https://dblp.org}
}

@inproceedings{gkd,
  author       = {Rishabh Agarwal and
                  Nino Vieillard and
                  Yongchao Zhou and
                  Piotr Stanczyk and
                  Sabela Ramos Garea and
                  Matthieu Geist and
                  Olivier Bachem},
  title        = {On-Policy Distillation of Language Models: Learning from Self-Generated
                  Mistakes},
  booktitle    = {The Twelfth International Conference on Learning Representations,
                  {ICLR} 2024, Vienna, Austria, May 7-11, 2024},
  publisher    = {OpenReview.net},
  year         = {2024},
  url          = {https://openreview.net/forum?id=3zKtaqxLhW},
  bibsource    = {dblp computer science bibliography, https://dblp.org}
}

@inproceedings{minillm,
  author       = {Yuxian Gu and
                  Li Dong and
                  Furu Wei and
                  Minlie Huang},
  title        = {MiniLLM: Knowledge Distillation of Large Language Models},
  booktitle    = {The Twelfth International Conference on Learning Representations,
                  {ICLR} 2024, Vienna, Austria, May 7-11, 2024},
  publisher    = {OpenReview.net},
  year         = {2024},
  url          = {https://openreview.net/forum?id=5h0qf7IBZZ},
  bibsource    = {dblp computer science bibliography, https://dblp.org}
}

@article{filter-opd,
  author       = {Yuying Li and
                  Leqi Zheng and
                  Yongzi Yu and
                  Wenrui Zhou and
                  Xuchang Zhong and
                  Xing Hu and
                  Jing Jin and
                  Hangjie Yuan and
                  Tao Feng},
  title        = {Filter, Then Reweight: Rethinking Optimization Granularity in On-Policy
                  Distillation},
  journal      = {CoRR},
  volume       = {abs/2606.02684},
  year         = {2026},
  url          = {https://doi.org/10.48550/arXiv.2606.02684},
  doi          = {10.48550/ARXIV.2606.02684},
  eprinttype   = {arXiv},
  eprint       = {2606.02684},
  bibsource    = {dblp computer science bibliography, https://dblp.org}
}

@article{entropy,
  author       = {Woogyeol Jin and
                  Taywon Min and
                  Yongjin Yang and
                  Swanand Ravindra Kadhe and
                  Yi Zhou and
                  Dennis Wei and
                  Nathalie Baracaldo and
                  Kimin Lee},
  title        = {Entropy-Aware On-Policy Distillation of Language Models},
  journal      = {CoRR},
  volume       = {abs/2603.07079},
  year         = {2026},
  url          = {https://doi.org/10.48550/arXiv.2603.07079},
  doi          = {10.48550/ARXIV.2603.07079},
  eprinttype   = {arXiv},
  eprint       = {2603.07079},
  bibsource    = {dblp computer science bibliography, https://dblp.org}
}

@inproceedings{promptKD,
  author       = {Gyeongman Kim and
                  Doohyuk Jang and
                  Eunho Yang},
  editor       = {Yaser Al{-}Onaizan and
                  Mohit Bansal and
                  Yun{-}Nung Chen},
  title        = {PromptKD: Distilling Student-Friendly Knowledge for Generative Language
                  Models via Prompt Tuning},
  booktitle    = {Findings of the Association for Computational Linguistics: {EMNLP}
                  2024, Miami, Florida, USA, November 12-16, 2024},
  series       = {Findings of {ACL}},
  volume       = {{EMNLP} 2024},
  pages        = {6266--6282},
  publisher    = {Association for Computational Linguistics},
  year         = {2024},
  url          = {https://doi.org/10.18653/v1/2024.findings-emnlp.364},
  doi          = {10.18653/V1/2024.FINDINGS-EMNLP.364},
  bibsource    = {dblp computer science bibliography, https://dblp.org}
}

@inproceedings{speculative-opd,
  author       = {Wenda Xu and
                  Rujun Han and
                  Zifeng Wang and
                  Long T. Le and
                  Dhruv Madeka and
                  Lei Li and
                  William Yang Wang and
                  Rishabh Agarwal and
                  Chen{-}Yu Lee and
                  Tomas Pfister},
  title        = {Speculative Knowledge Distillation: Bridging the Teacher-Student Gap
                  Through Interleaved Sampling},
  booktitle    = {The Thirteenth International Conference on Learning Representations,
                  {ICLR} 2025, Singapore, April 24-28, 2025},
  publisher    = {OpenReview.net},
  year         = {2025},
  url          = {https://openreview.net/forum?id=EgJhwYR2tB},
  bibsource    = {dblp computer science bibliography, https://dblp.org}
}

@article{self-adaptive,
  author       = {Adam Zweiger and
                  Jyothish Pari and
                  Han Guo and
                  Ekin Aky{\"{u}}rek and
                  Yoon Kim and
                  Pulkit Agrawal},
  title        = {Self-Adapting Language Models},
  journal      = {CoRR},
  volume       = {abs/2506.10943},
  year         = {2025},
  url          = {https://doi.org/10.48550/arXiv.2506.10943},
  doi          = {10.48550/ARXIV.2506.10943},
  eprinttype   = {arXiv},
  eprint       = {2506.10943},
  bibsource    = {dblp computer science bibliography, https://dblp.org}
}

@inproceedings{self-instruct,
  author       = {Yizhong Wang and
                  Yeganeh Kordi and
                  Swaroop Mishra and
                  Alisa Liu and
                  Noah A. Smith and
                  Daniel Khashabi and
                  Hannaneh Hajishirzi},
  editor       = {Anna Rogers and
                  Jordan L. Boyd{-}Graber and
                  Naoaki Okazaki},
  title        = {Self-Instruct: Aligning Language Models with Self-Generated Instructions},
  booktitle    = {Proceedings of the 61st Annual Meeting of the Association for Computational
                  Linguistics (Volume 1: Long Papers), {ACL} 2023, Toronto, Canada,
                  July 9-14, 2023},
  pages        = {13484--13508},
  publisher    = {Association for Computational Linguistics},
  year         = {2023},
  url          = {https://doi.org/10.18653/v1/2023.acl-long.754},
  doi          = {10.18653/V1/2023.ACL-LONG.754},
  bibsource    = {dblp computer science bibliography, https://dblp.org}
}

@article{grpo,
  author       = {Zhihong Shao and
                  Peiyi Wang and
                  Qihao Zhu and
                  Runxin Xu and
                  Junxiao Song and
                  Mingchuan Zhang and
                  Y. K. Li and
                  Y. Wu and
                  Daya Guo},
  title        = {DeepSeekMath: Pushing the Limits of Mathematical Reasoning in Open
                  Language Models},
  journal      = {CoRR},
  volume       = {abs/2402.03300},
  year         = {2024},
  url          = {https://doi.org/10.48550/arXiv.2402.03300},
  doi          = {10.48550/ARXIV.2402.03300},
  eprinttype   = {arXiv},
  eprint       = {2402.03300},
  bibsource    = {dblp computer science bibliography, https://dblp.org}
}

@article{qwen3-embedding,
  author       = {Yanzhao Zhang and
                  Mingxin Li and
                  Dingkun Long and
                  Xin Zhang and
                  Huan Lin and
                  Baosong Yang and
                  Pengjun Xie and
                  An Yang and
                  Dayiheng Liu and
                  Junyang Lin and
                  Fei Huang and
                  Jingren Zhou},
  title        = {Qwen3 Embedding: Advancing Text Embedding and Reranking Through Foundation
                  Models},
  journal      = {CoRR},
  volume       = {abs/2506.05176},
  year         = {2025},
  url          = {https://doi.org/10.48550/arXiv.2506.05176},
  doi          = {10.48550/ARXIV.2506.05176},
  eprinttype   = {arXiv},
  eprint       = {2506.05176},
  bibsource    = {dblp computer science bibliography, https://dblp.org}
}

@inproceedings{math500,
  author       = {Dan Hendrycks and
                  Collin Burns and
                  Saurav Kadavath and
                  Akul Arora and
                  Steven Basart and
                  Eric Tang and
                  Dawn Song and
                  Jacob Steinhardt},
  editor       = {Joaquin Vanschoren and
                  Sai{-}Kit Yeung},
  title        = {Measuring Mathematical Problem Solving With the {MATH} Dataset},
  booktitle    = {Proceedings of the Neural Information Processing Systems Track on
                  Datasets and Benchmarks 1, NeurIPS Datasets and Benchmarks 2021, December
                  2021, virtual},
  year         = {2021},
  url          = {https://datasets-benchmarks-proceedings.neurips.cc/paper/2021/hash/be83ab3ecd0db773eb2dc1b0a17836a1-Abstract-round2.html},
  bibsource    = {dblp computer science bibliography, https://dblp.org}
}

@article{qwen3,
  author       = {Qwen Team},
  title        = {Qwen3 Technical Report},
  journal      = {CoRR},
  volume       = {abs/2505.09388},
  year         = {2025},
  url          = {https://doi.org/10.48550/arXiv.2505.09388},
  doi          = {10.48550/ARXIV.2505.09388},
  eprinttype   = {arXiv},
  eprint       = {2505.09388},
  bibsource    = {dblp computer science bibliography, https://dblp.org}
}

@article{qwen2.5-coder,
  author       = {Binyuan Hui and
                  Jian Yang and
                  Zeyu Cui and
                  Jiaxi Yang and
                  Dayiheng Liu and
                  Lei Zhang and
                  Tianyu Liu and
                  Jiajun Zhang and
                  Bowen Yu and
                  Kai Dang and
                  An Yang and
                  Rui Men and
                  Fei Huang and
                  Xingzhang Ren and
                  Xuancheng Ren and
                  Jingren Zhou and
                  Junyang Lin},
  title        = {Qwen2.5-Coder Technical Report},
  journal      = {CoRR},
  volume       = {abs/2409.12186},
  year         = {2024},
  url          = {https://doi.org/10.48550/arXiv.2409.12186},
  doi          = {10.48550/ARXIV.2409.12186},
  eprinttype   = {arXiv},
  eprint       = {2409.12186},
  bibsource    = {dblp computer science bibliography, https://dblp.org}
}

@article{deepseek-r1,
  author       = {Daya Guo and
                  Dejian Yang and
                  Haowei Zhang and
                  Junxiao Song and
                  Peiyi Wang and
                  Qihao Zhu and
                  Runxin Xu and
                  Ruoyu Zhang and
                  Shirong Ma and
                  Xiao Bi and
                  Xiaokang Zhang and
                  Xingkai Yu and
                  Yu Wu and
                  Z. F. Wu and
                  Zhibin Gou and
                  Zhihong Shao and
                  Zhuoshu Li and
                  Ziyi Gao and
                  Aixin Liu and
                  Bing Xue and
                  Bingxuan Wang and
                  Bochao Wu and
                  Bei Feng and
                  Chengda Lu and
                  Chenggang Zhao and
                  Chengqi Deng and
                  Chong Ruan and
                  Damai Dai and
                  Deli Chen and
                  Dongjie Ji and
                  Erhang Li and
                  Fangyun Lin and
                  Fucong Dai and
                  Fuli Luo and
                  Guangbo Hao and
                  Guanting Chen and
                  Guowei Li and
                  Hao Zhang and
                  Hanwei Xu and
                  Honghui Ding and
                  Huazuo Gao and
                  Hui Qu and
                  Hui Li and
                  Jianzhong Guo and
                  Jiashi Li and
                  Jingchang Chen and
                  Jingyang Yuan and
                  Jinhao Tu and
                  Junjie Qiu and
                  Junlong Li and
                  J. L. Cai and
                  Jiaqi Ni and
                  Jian Liang and
                  Jin Chen and
                  Kai Dong and
                  Kai Hu and
                  Kaichao You and
                  Kaige Gao and
                  Kang Guan and
                  Kexin Huang and
                  Kuai Yu and
                  Lean Wang and
                  Lecong Zhang and
                  Liang Zhao and
                  Litong Wang and
                  Liyue Zhang and
                  Lei Xu and
                  Leyi Xia and
                  Mingchuan Zhang and
                  Minghua Zhang and
                  Minghui Tang and
                  Mingxu Zhou and
                  Meng Li and
                  Miaojun Wang and
                  Mingming Li and
                  Ning Tian and
                  Panpan Huang and
                  Peng Zhang and
                  Qiancheng Wang and
                  Qinyu Chen and
                  Qiushi Du and
                  Ruiqi Ge and
                  Ruisong Zhang and
                  Ruizhe Pan and
                  Runji Wang and
                  R. J. Chen and
                  R. L. Jin and
                  Ruyi Chen and
                  Shanghao Lu and
                  Shangyan Zhou and
                  Shanhuang Chen and
                  Shengfeng Ye and
                  Shiyu Wang and
                  Shuiping Yu and
                  Shunfeng Zhou and
                  Shuting Pan and
                  S. S. Li and
                  Shuang Zhou and
                  Shaoqing Wu and
                  Tao Yun and
                  Tian Pei and
                  Tianyu Sun and
                  Tao Wang and
                  Wangding Zeng and
                  Wen Liu and
                  Wenfeng Liang and
                  Wenjun Gao and
                  Wenqin Yu and
                  Wentao Zhang and
                  W. L. Xiao and
                  Wei An and
                  Xiaodong Liu and
                  Xiaohan Wang and
                  Xiaokang Chen and
                  Xiaotao Nie and
                  Xin Cheng and
                  Xin Liu and
                  Xin Xie and
                  Xingchao Liu and
                  Xinyu Yang and
                  Xinyuan Li and
                  Xuecheng Su and
                  Xuheng Lin and
                  X. Q. Li and
                  Xiangyue Jin and
                  Xiaojin Shen and
                  Xiaosha Chen and
                  Xiaowen Sun and
                  Xiaoxiang Wang and
                  Xinnan Song and
                  Xinyi Zhou and
                  Xianzu Wang and
                  Xinxia Shan and
                  Y. K. Li and
                  Y. Q. Wang and
                  Y. X. Wei and
                  Yang Zhang and
                  Yanhong Xu and
                  Yao Li and
                  Yao Zhao and
                  Yaofeng Sun and
                  Yaohui Wang and
                  Yi Yu and
                  Yichao Zhang and
                  Yifan Shi and
                  Yiliang Xiong and
                  Ying He and
                  Yishi Piao and
                  Yisong Wang and
                  Yixuan Tan and
                  Yiyang Ma and
                  Yiyuan Liu and
                  Yongqiang Guo and
                  Yuan Ou and
                  Yuduan Wang and
                  Yue Gong and
                  Yuheng Zou and
                  Yujia He and
                  Yunfan Xiong and
                  Yuxiang Luo and
                  Yuxiang You and
                  Yuxuan Liu and
                  Yuyang Zhou and
                  Y. X. Zhu and
                  Yanping Huang and
                  Yaohui Li and
                  Yi Zheng and
                  Yuchen Zhu and
                  Yunxian Ma and
                  Ying Tang and
                  Yukun Zha and
                  Yuting Yan and
                  Z. Z. Ren and
                  Zehui Ren and
                  Zhangli Sha and
                  Zhe Fu and
                  Zhean Xu and
                  Zhenda Xie and
                  Zhengyan Zhang and
                  Zhewen Hao and
                  Zhicheng Ma and
                  Zhigang Yan and
                  Zhiyu Wu and
                  Zihui Gu and
                  Zijia Zhu and
                  Zijun Liu and
                  Zilin Li and
                  Ziwei Xie and
                  Ziyang Song and
                  Zizheng Pan and
                  Zhen Huang and
                  Zhipeng Xu and
                  Zhongyu Zhang and
                  Zhen Zhang},
  title        = {DeepSeek-R1 incentivizes reasoning in LLMs through reinforcement learning},
  journal      = {Nat.},
  volume       = {645},
  number       = {8081},
  pages        = {633--638},
  year         = {2025},
  url          = {https://doi.org/10.1038/s41586-025-09422-z},
  doi          = {10.1038/S41586-025-09422-Z},
  bibsource    = {dblp computer science bibliography, https://dblp.org}
}

@article{skywork,
  author       = {Jujie He and
                  Jiacai Liu and
                  Chris Yuhao Liu and
                  Rui Yan and
                  Chaojie Wang and
                  Peng Cheng and
                  Xiaoyu Zhang and
                  Fuxiang Zhang and
                  Jiacheng Xu and
                  Wei Shen and
                  Siyuan Li and
                  Liang Zeng and
                  Tianwen Wei and
                  Cheng Cheng and
                  Bo An and
                  Yang Liu and
                  Yahui Zhou},
  title        = {Skywork Open Reasoner 1 Technical Report},
  journal      = {CoRR},
  volume       = {abs/2505.22312},
  year         = {2025},
  url          = {https://doi.org/10.48550/arXiv.2505.22312},
  doi          = {10.48550/ARXIV.2505.22312},
  eprinttype   = {arXiv},
  eprint       = {2505.22312},
  bibsource    = {dblp computer science bibliography, https://dblp.org}
}

@article{qwen3-4b-non-thinking-rl-math,
  author       = {Wenkai Yang and
                  Weijie Liu and
                  Ruobing Xie and
                  Kai Yang and
                  Saiyong Yang and
                  Yankai Lin},
  title        = {Learning beyond Teacher: Generalized On-Policy Distillation with Reward
                  Extrapolation},
  journal      = {CoRR},
  volume       = {abs/2602.12125},
  year         = {2026},
  url          = {https://doi.org/10.48550/arXiv.2602.12125},
  doi          = {10.48550/ARXIV.2602.12125},
  eprinttype   = {arXiv},
  eprint       = {2602.12125},
  bibsource    = {dblp computer science bibliography, https://dblp.org}
}

@article{tau2,
  author       = {Victor Barr{\`{e}}s and
                  Honghua Dong and
                  Soham Ray and
                  Xujie Si and
                  Karthik Narasimhan},
  title        = {{\(\tau\)}\({}^{\mbox{2}}\)-Bench: Evaluating Conversational Agents
                  in a Dual-Control Environment},
  journal      = {CoRR},
  volume       = {abs/2506.07982},
  year         = {2025},
  url          = {https://doi.org/10.48550/arXiv.2506.07982},
  doi          = {10.48550/ARXIV.2506.07982},
  eprinttype   = {arXiv},
  eprint       = {2506.07982},
  bibsource    = {dblp computer science bibliography, https://dblp.org}
}

@misc{aime25,
      title={American Invitational Mathematics Examination (AIME) 2025}, 
      author={Zhang, Yifan and Math-AI, Team},
      year={2025},
}

@article{codecontests,
  author       = {Yujia Li and
                  David H. Choi and
                  Junyoung Chung and
                  Nate Kushman and
                  Julian Schrittwieser and
                  R{\'{e}}mi Leblond and
                  Tom Eccles and
                  James Keeling and
                  Felix Gimeno and
                  Agustin Dal Lago and
                  Thomas Hubert and
                  Peter Choy and
                  Cyprien de Masson d'Autume and
                  Igor Babuschkin and
                  Xinyun Chen and
                  Po{-}Sen Huang and
                  Johannes Welbl and
                  Sven Gowal and
                  Alexey Cherepanov and
                  James Molloy and
                  Daniel J. Mankowitz and
                  Esme Sutherland Robson and
                  Pushmeet Kohli and
                  Nando de Freitas and
                  Koray Kavukcuoglu and
                  Oriol Vinyals},
  title        = {Competition-Level Code Generation with AlphaCode},
  journal      = {CoRR},
  volume       = {abs/2203.07814},
  year         = {2022},
  url          = {https://doi.org/10.48550/arXiv.2203.07814},
  doi          = {10.48550/ARXIV.2203.07814},
  eprinttype   = {arXiv},
  eprint       = {2203.07814},
  bibsource    = {dblp computer science bibliography, https://dblp.org}
}

@inproceedings{apps,
  author       = {Dan Hendrycks and
                  Steven Basart and
                  Saurav Kadavath and
                  Mantas Mazeika and
                  Akul Arora and
                  Ethan Guo and
                  Collin Burns and
                  Samir Puranik and
                  Horace He and
                  Dawn Song and
                  Jacob Steinhardt},
  editor       = {Joaquin Vanschoren and
                  Sai{-}Kit Yeung},
  title        = {Measuring Coding Challenge Competence With {APPS}},
  booktitle    = {Proceedings of the Neural Information Processing Systems Track on
                  Datasets and Benchmarks 1, NeurIPS Datasets and Benchmarks 2021, December
                  2021, virtual},
  year         = {2021},
  url          = {https://datasets-benchmarks-proceedings.neurips.cc/paper/2021/hash/c24cd76e1ce41366a4bbe8a49b02a028-Abstract-round2.html},
  bibsource    = {dblp computer science bibliography, https://dblp.org}
}

@inproceedings{mbpp,
  author       = {Jiawei Liu and
                  Chunqiu Steven Xia and
                  Yuyao Wang and
                  Lingming Zhang},
  editor       = {Alice Oh and
                  Tristan Naumann and
                  Amir Globerson and
                  Kate Saenko and
                  Moritz Hardt and
                  Sergey Levine},
  title        = {Is Your Code Generated by ChatGPT Really Correct? Rigorous Evaluation
                  of Large Language Models for Code Generation},
  booktitle    = {Advances in Neural Information Processing Systems 36: Annual Conference
                  on Neural Information Processing Systems 2023, NeurIPS 2023, New Orleans,
                  LA, USA, December 10 - 16, 2023},
  year         = {2023},
  url          = {http://papers.nips.cc/paper\_files/paper/2023/hash/43e9d647ccd3e4b7b5baab53f0368686-Abstract-Conference.html},
  bibsource    = {dblp computer science bibliography, https://dblp.org}
}

@inproceedings{lcb,
  author       = {Naman Jain and
                  King Han and
                  Alex Gu and
                  Wen{-}Ding Li and
                  Fanjia Yan and
                  Tianjun Zhang and
                  Sida Wang and
                  Armando Solar{-}Lezama and
                  Koushik Sen and
                  Ion Stoica},
  title        = {LiveCodeBench: Holistic and Contamination Free Evaluation of Large
                  Language Models for Code},
  booktitle    = {The Thirteenth International Conference on Learning Representations,
                  {ICLR} 2025, Singapore, April 24-28, 2025},
  publisher    = {OpenReview.net},
  year         = {2025},
  url          = {https://openreview.net/forum?id=chfJJYC3iL},
  bibsource    = {dblp computer science bibliography, https://dblp.org}
}

@article{brumo2025_aime26,
  author       = {Jasper Dekoninck and
                  Nikola Jovanovic and
                  Tim Gehrunger and
                  K{\'{a}}ri R{\"{o}}gnvaldsson and
                  Ivo Petrov and
                  Chenhao Sun and
                  Martin T. Vechev},
  title        = {Beyond Benchmarks: MathArena as an Evaluation Platform for Mathematics
                  with LLMs},
  journal      = {CoRR},
  volume       = {abs/2605.00674},
  year         = {2026},
  url          = {https://doi.org/10.48550/arXiv.2605.00674},
  doi          = {10.48550/ARXIV.2605.00674},
  eprinttype   = {arXiv},
  eprint       = {2605.00674},
  bibsource    = {dblp computer science bibliography, https://dblp.org}
}

@misc{aime24,
  author = {{Mathematical Association of America}},
  title  = {{American Invitational Mathematics Examination (AIME) 2024}},
  year   = {2024},
  url    = {https://maa.org/maa-invitational-competitions/}
}

@misc{amc23,
  author = {{Mathematical Association of America}},
  title  = {{American Mathematics Competitions (AMC) 2023}},
  year   = {2023},
  url    = {https://maa.org/student-programs/amc/}
}

@inproceedings{math500_2,
  author       = {Hunter Lightman and
                  Vineet Kosaraju and
                  Yuri Burda and
                  Harrison Edwards and
                  Bowen Baker and
                  Teddy Lee and
                  Jan Leike and
                  John Schulman and
                  Ilya Sutskever and
                  Karl Cobbe},
  title        = {Let's Verify Step by Step},
  booktitle    = {The Twelfth International Conference on Learning Representations,
                  {ICLR} 2024, Vienna, Austria, May 7-11, 2024},
  publisher    = {OpenReview.net},
  year         = {2024},
  url          = {https://openreview.net/forum?id=v8L0pN6EOi},
  bibsource    = {dblp computer science bibliography, https://dblp.org}
}

@article{mbpp+,
  author       = {Jacob Austin and
                  Augustus Odena and
                  Maxwell I. Nye and
                  Maarten Bosma and
                  Henryk Michalewski and
                  David Dohan and
                  Ellen Jiang and
                  Carrie J. Cai and
                  Michael Terry and
                  Quoc V. Le and
                  Charles Sutton},
  title        = {Program Synthesis with Large Language Models},
  journal      = {CoRR},
  volume       = {abs/2108.07732},
  year         = {2021},
  url          = {https://doi.org/10.48550/arXiv.2108.07732},
  doi          = {10.48550/ARXIV.2108.07732},
  eprinttype   = {arXiv},
  eprint       = {2108.07732},
  bibsource    = {dblp computer science bibliography, https://dblp.org}
}

@article{ppo,
  author       = {John Schulman and
                  Filip Wolski and
                  Prafulla Dhariwal and
                  Alec Radford and
                  Oleg Klimov},
  title        = {Proximal Policy Optimization Algorithms},
  journal      = {CoRR},
  volume       = {abs/1707.06347},
  year         = {2017},
  url          = {https://doi.org/10.48550/arXiv.1707.06347},
  doi          = {10.48550/ARXIV.1707.06347},
  eprinttype   = {arXiv},
  eprint       = {1707.06347},
  bibsource    = {dblp computer science bibliography, https://dblp.org}
}

@inproceedings{vllm,
  author       = {Woosuk Kwon and
                  Zhuohan Li and
                  Siyuan Zhuang and
                  Ying Sheng and
                  Lianmin Zheng and
                  Cody Hao Yu and
                  Joseph Gonzalez and
                  Hao Zhang and
                  Ion Stoica},
  editor       = {Jason Flinn and
                  Margo I. Seltzer and
                  Peter Druschel and
                  Antoine Kaufmann and
                  Jonathan Mace},
  title        = {Efficient Memory Management for Large Language Model Serving with
                  PagedAttention},
  booktitle    = {Proceedings of the 29th Symposium on Operating Systems Principles,
                  {SOSP} 2023, Koblenz, Germany, October 23-26, 2023},
  pages        = {611--626},
  publisher    = {{ACM}},
  year         = {2023},
  url          = {https://doi.org/10.1145/3600006.3613165},
  doi          = {10.1145/3600006.3613165},
  bibsource    = {dblp computer science bibliography, https://dblp.org}
}

@inproceedings{verl,
  author       = {Guangming Sheng and
                  Chi Zhang and
                  Zilingfeng Ye and
                  Xibin Wu and
                  Wang Zhang and
                  Ru Zhang and
                  Yanghua Peng and
                  Haibin Lin and
                  Chuan Wu},
  title        = {HybridFlow: {A} Flexible and Efficient {RLHF} Framework},
  booktitle    = {Proceedings of the Twentieth European Conference on Computer Systems,
                  EuroSys 2025, Rotterdam, The Netherlands, 30 March 2025 - 3 April
                  2025},
  pages        = {1279--1297},
  publisher    = {{ACM}},
  year         = {2025},
  url          = {https://doi.org/10.1145/3689031.3696075},
  doi          = {10.1145/3689031.3696075},
  bibsource    = {dblp computer science bibliography, https://dblp.org}
}

@article{sympy,
  author       = {Aaron Meurer and
                  Christopher P. Smith and
                  Mateusz Paprocki and
                  Ondrej Cert{\'{\i}}k and
                  Sergey B. Kirpichev and
                  Matthew Rocklin and
                  Amit Kumar and
                  Sergiu Ivanov and
                  Jason Keith Moore and
                  Sartaj Singh and
                  Thilina Rathnayake and
                  Sean Vig and
                  Brian E. Granger and
                  Richard P. Muller and
                  Francesco Bonazzi and
                  Harsh Gupta and
                  Shivam Vats and
                  Fredrik Johansson and
                  Fabian Pedregosa and
                  Matthew J. Curry and
                  Andy R. Terrel and
                  Step{\'{a}}n Roucka and
                  Ashutosh Saboo and
                  Isuru Fernando and
                  Sumith Kulal and
                  Robert Cimrman and
                  Anthony M. Scopatz},
  title        = {SymPy: symbolic computing in Python},
  journal      = {PeerJ Comput. Sci.},
  volume       = {3},
  pages        = {e103},
  year         = {2017},
  url          = {https://doi.org/10.7717/peerj-cs.103},
  doi          = {10.7717/PEERJ-CS.103},
  bibsource    = {dblp computer science bibliography, https://dblp.org}
}

@article{justrl,
  author       = {Bingxiang He and
                  Zekai Qu and
                  Zeyuan Liu and
                  Yinghao Chen and
                  Yuxin Zuo and
                  Cheng Qian and
                  Kaiyan Zhang and
                  Weize Chen and
                  Chaojun Xiao and
                  Ganqu Cui and
                  Ning Ding and
                  Zhiyuan Liu},
  title        = {JustRL: Scaling a 1.5B {LLM} with a Simple {RL} Recipe},
  journal      = {CoRR},
  volume       = {abs/2512.16649},
  year         = {2025},
  url          = {https://doi.org/10.48550/arXiv.2512.16649},
  doi          = {10.48550/ARXIV.2512.16649},
  eprinttype   = {arXiv},
  eprint       = {2512.16649},
  bibsource    = {dblp computer science bibliography, https://dblp.org}
}

@misc{openai2024gpt4omini,
  author = {{OpenAI}},
  title  = {{GPT-4o mini}: Advancing Cost-Efficient Intelligence},
  year   = {2024},
  month  = jul,
  url    = {https://openai.com/index/gpt-4o-mini-advancing-cost-efficient-intelligence/}
}
\bibliographystyle{icml2026}

\newpage
\appendix
\onecolumn
\section{Experimental Details}

\subsection{Algorithm}
Algorithm~\ref{alg:ask_to_distill} summarizes our method.

\begin{algorithm*}[t]
\caption{\method: Learning to Elicit Teacher Supervision}
\label{alg:ask_to_distill}
\small
\begin{algorithmic}[1]
\REQUIRE Student $\pi_S$, teacher $\pi_T$, question-generation instruction $x$;
candidate size $K$, question budget $L_q$, shared response budget $L_{\max}$;
calibration percentile $\kappa$ and question count $N_{\mathrm{cal}}$;
synthesizer steps $T_{\mathrm{syn}}$, groups per step $B$, group size $g$;
raw question count $N$, OPD epochs $E$, responses per question $n_{\mathrm{roll}}$.
\ENSURE Synthetic question bank $\mathcal{D}$ and distilled student $\pi_S$.

\STATE Initialize synthesizer $S_\theta \leftarrow \operatorname{Copy}(\pi_S)$.
\STATE Freeze $\pi_S$ and $\pi_T$ during synthesizer optimization.
\STATE Calibrate $\delta_{\mathrm{ref}}$ as the $\kappa$-th percentile of token discrepancies
on student rollouts for $N_{\mathrm{cal}}$ student-generated questions.

\STATE \textbf{Stage 1: Optimize the synthetic question distribution}
\FOR{$k=1,\ldots,T_{\mathrm{syn}}$}
    \STATE Sample $B$ groups of questions
    $\mathcal{Q}_b=\{q_{b,i}\}_{i=1}^{g}$ from $S_\theta(\cdot\mid x)$
    under budget $L_q$.
    \STATE Form the question batch $\mathcal{Q}$ by concatenating all groups.
    \STATE Cluster questions using BLEU-based pairwise distances
    to compute $\operatorname{share}(q)$ for each $q\in\mathcal{Q}$.
    \FORALL{$q\in\mathcal{Q}$}
        \STATE Sample $y\sim\pi_S(\cdot\mid q)$ under budget $L_{\max}$;
        let $L=|y|$.
        \STATE At each position $t$, evaluate
        $p_S=\pi_S(\cdot\mid q,y_{<t})$ and
        $p_T=\pi_T(\cdot\mid q,y_{<t})$.
        \STATE Set
        $A_t=\operatorname{TopK}(p_S)$,
        $B_t=\operatorname{TopK}(p_T)$, and
        $U_t=A_t\cup B_t$.
        \STATE Normalize $p_S$ and $p_T$ over $U_t$ to obtain
        $\bar p_S$ and $\bar p_T$.
        \STATE Compute
        $\delta_t=\operatorname{KL}(\bar p_T\|\bar p_S)$,
        $\omega_t=\sum_{v\in A_t}p_T(v)$, and
        $s_t=\min\!\left(\delta_t/(\delta_{\mathrm{ref}}+\epsilon),1\right)\omega_t$.
        \STATE Compute
        $P_{\mathrm{len}}(q)=1-\lambda_{\mathrm{len}}\min(L/L_{\max},1)$
        and
        $P_{\mathrm{div}}(q)=\max(0,1-\lambda_{\mathrm{div}}\operatorname{share}(q))$.
        \STATE Assign question reward
        $R(q)=P_{\mathrm{len}}(q)P_{\mathrm{div}}(q)L^{-\alpha}
        \sum_{t=1}^{L}s_t$.
    \ENDFOR
    \STATE Compute reward mean $\mu_b$ and standard deviation $\sigma_b$
    within each group $\mathcal{Q}_b$.
    \STATE Set
    $\widehat{A}_{b,i}=(R(q_{b,i})-\mu_b)/(\sigma_b+\epsilon)$.
    \STATE Update only $S_\theta$ with GRPO, applying
    $\widehat{A}_{b,i}$ to the generated tokens of question $q_{b,i}$.
    \STATE Save synthesizer checkpoints at the prescribed interval.
\ENDFOR

\STATE \textbf{Stage 2: Construct the synthetic question bank}
\STATE Select a synthesizer checkpoint $S_{\theta^\star}$.
\STATE Sample $N$ questions from $S_{\theta^\star}(\cdot\mid x)$
under budget $L_q$.
\STATE Apply format filtering and remove malformed, incomplete,
or answer-leaking outputs.
\STATE Apply lexical deduplication using token-set Jaccard similarity.
\STATE Apply semantic deduplication using the frozen Qwen3-Embedding-4B model.
\STATE Store the retained questions as the prompt-only bank $\mathcal{D}$.

\STATE \textbf{Stage 3: Distill the student on the synthetic bank}
\STATE Unfreeze $\pi_S$ and keep $\pi_T$ frozen.
\FOR{$e=1,\ldots,E$}
    \FORALL{minibatches $\mathcal{B}\subset\mathcal{D}$}
        \STATE Snapshot the current student as $\pi_{\mathrm{old}}$.
        \STATE For each $q\in\mathcal{B}$, sample $n_{\mathrm{roll}}$ responses
        from $\pi_{\mathrm{old}}(\cdot\mid q)$ under budget $L_{\max}$.
        \STATE At each response position $t$, evaluate
        $p_S=\pi_{\mathrm{old}}(\cdot\mid q,y_{<t})$ and
        $p_T=\pi_T(\cdot\mid q,y_{<t})$.
        \STATE Set $A_t=\operatorname{TopK}(p_S)$ and
        $w_t(v)=p_S(v)/\sum_{u\in A_t}p_S(u)$ for $v\in A_t$.
        \STATE Compute per-candidate advantages
        $r_t(v)=w_t(v)\log\!\left(p_T(v)/p_S(v)\right)$,
        $v\in A_t$.
        \STATE Update $\pi_S$ with the PPO-clipped objective using
        $r_t(v)$ directly as fixed advantages,
        summing over candidates and averaging over valid response tokens.
    \ENDFOR
\ENDFOR
\STATE \textbf{return} $\mathcal{D},\pi_S$.
\end{algorithmic}
\end{algorithm*}

\subsection{Implementation Details}
\label{sec:implementation}

Our framework was developed based on verl \cite{verl}. For reward computation, we use $K=16$ candidate tokens, a question-generation budget of $L_q=4096$ tokens, and a student-response budget of $L_{\max}=7168$ tokens, with the latter matched to the downstream OPD response budget by construction (\S\ref{sec:opd}). For each student--teacher pair, the discrepancy scale $\delta_{\mathrm{ref}}$ is set to the $\kappa=95$th percentile of $\delta_t$ estimated from $N_{\mathrm{cal}}=200$ questions, and we use $\epsilon=10^{-8}$ in Eq.~\eqref{eq:st}. For reward shaping, we set the aggregation exponent $\alpha=0.5$, the length-penalty coefficient $\lambda_{\mathrm{len}}=0.3$, the diversity-penalty coefficient $\lambda_{\mathrm{div}}=1.0$, and the clustering threshold $\tau=0.1$. The synthetic-data synthesizer is optimized with GRPO using 8 prompts and a group size of $g=8$, yielding 64 candidate questions per step. We use a learning rate of $1\times10^{-6}$, entropy and KL coefficients of 0.01 and 0.005, respectively, and a rollout temperature of 1.0, training for 60 steps and saving checkpoints every 5 steps. We generate $N=20{,}000$ raw questions in 10 shards of 2000 samples each and apply Jaccard and cosine deduplication thresholds of $\tau_{\mathrm{lex}}=0.7$ and $\tau_{\mathrm{sem}}=0.95$, respectively. Downstream OPD uses token-level rewards directly as advantages, with the student's top-$K$ candidates ($K=16$) weighted by $\bar p_S$. We sample $n_{\mathrm{roll}}=4$ responses per prompt at temperature 1.0 and train with a batch size of 64, a learning rate of $1\times10^{-6}$, and \texttt{fp32} parameters for 2 epochs, saving checkpoints every 5 steps. For each student--teacher pair and each OPD method, we select a single downstream student checkpoint based on AMC23 performance for mathematical reasoning and LiveCodeBench performance for code generation. All benchmark scores reported for that method are evaluated at this same checkpoint, rather than selected independently for each benchmark. AMC23 and LiveCodeBench therefore serve as checkpoint-selection benchmarks; their reported scores are not held-out estimates.

\subsection{Evaluation Details}
\label{sec:eval_details}

\paragraph{Metric.}
For each benchmark, we sample $k$ responses per problem and report avg@$k$: the fraction of correct responses for each problem, averaged over problems and scaled to $[0,100]$. With $k=1$, this reduces to single-sample accuracy. All reported results are means over three independent runs with different random seeds. The reported averages weight each benchmark equally, regardless of its number of problems. Table~\ref{tab:eval_settings} lists the per-benchmark settings. All evaluations use vLLM (v0.11.0) \cite{vllm} with bf16 inference on NVIDIA A100 80GB GPUs.

\paragraph{Mathematical reasoning.}
We use the DAPO prompt format, which asks the model to reason step by step and to give its final answer within \texttt{\textbackslash boxed\{\}}. Qwen3 students and teachers are evaluated in non-thinking mode, whereas DeepSeek-R1-Distill students and their teachers (JustRL and Skywork-OR1) use their default thinking templates. We extract the last \texttt{\textbackslash boxed\{\}} expression, falling back to the last ``Answer:'' line, and verify it against the reference answer by rule-based string normalization and symbolic equivalence checking with SymPy \cite{sympy}; no model-based judge is used.

\paragraph{Code generation.}
For LiveCodeBench, we use the v5--v6 release subset of 342 problems published between September 2024 and April 2025, with the original LiveCodeBench prompt. For APPS, we use the full introductory and competition test splits (1{,}000 problems each) and 1{,}000 interview problems drawn from the test split with a fixed random seed, keeping at most 50 test cases per problem; prompts follow the LiveCodeBench standard-input template. MBPP+ uses EvalPlus v0.2.0 (378 problems) with its standard prompt template and greedy decoding. Qwen3 models are evaluated in non-thinking mode. We extract the last fenced code block from each response and execute it against all test cases; a response is counted as correct only if it passes every test within the time limit. In the math-to-code transfer experiment (Figure~\ref{fig:math_to_code}), the Qwen3 models use the non-thinking settings in Table~\ref{tab:eval_settings}, whereas DeepSeek-R1-Distill students and their mathematics teachers are evaluated in thinking mode with their default chat templates, using temperature 0.6, top-$p$ 0.95, top-$k$ 20, and a maximum generation length of 32{,}768 tokens on all three benchmarks.

\begin{table}[h]
  \centering
  \caption{Evaluation settings. $k$ is the number of sampled responses per problem; ``Max tokens'' is the maximum generation length.}
  \label{tab:eval_settings}
  \small
  \begin{tabular}{@{}l r r c c c r@{}}
    \toprule
    Benchmark & \#Problems & $k$ & Temperature & Top-$p$ & Top-$k$ & Max tokens \\
    \midrule
    AIME24 / AIME25 / AIME26 & 30 each & 16 & 0.7 & 0.95 & -- & 31{,}744 \\
    AMC23 & 40 & 16 & 0.7 & 0.95 & -- & 31{,}744 \\
    BRUMO25 & 30 & 16 & 0.7 & 0.95 & -- & 31{,}744 \\
    MATH500 & 500 & 4 & 0.7 & 0.95 & -- & 31{,}744 \\
    \midrule
    LiveCodeBench (v5--v6) & 342 & 4 & 0.7 & 0.8 & 20 & 8{,}192 \\
    APPS-Intro & 1{,}000 & 1 & 0.2 & 0.95 & 20 & 4{,}096 \\
    APPS-Interview & 1{,}000 & 1 & 0.2 & 0.95 & 20 & 4{,}096 \\
    APPS-Competition & 1{,}000 & 1 & 0.2 & 0.95 & 20 & 4{,}096 \\
    MBPP+ & 378 & 1 & \multicolumn{3}{c}{greedy} & 1{,}024 \\
    CodeContests & 165 & 1 & 0.2 & 0.95 & 20 & 4096 \\
    \bottomrule
  \end{tabular}
\end{table}

\subsection{Discussion}

\paragraph{Combining with filtering and OPD variants}
\label{sec:orthogonality}

\method learns a question distribution for a target student--teacher pair, while filtering determines which sampled questions are retained and the downstream OPD procedure determines how teacher supervision updates the student. These operations act at different stages of the pipeline. Consequently, learned question generation can be followed by additional filtering or selection rules, and the resulting synthetic bank can be used by alternative distillation procedures that accept prompt-only datasets. Our format filtering and deduplication already illustrate this separation: they curate the synthesizer's outputs before those questions are used for OPD.

The learning-value reward nevertheless connects prompt optimization to the downstream distillation procedure. Our support-aware score is designed around teacher supervision on the student's candidate support. A procedure using a different divergence, candidate support, or supervision allocation rule may benefit from a corresponding adaptation of the generation reward. We therefore view \method as a modular approach to constructing distillation prompts, with the effectiveness of any particular combination depending on how well the generation objective matches the downstream learning signal. Exploring such combinations is a natural direction for future work.

\begin{wrapfigure}{l}{0.48\textwidth}
    \centering
    \includegraphics[width=\linewidth]{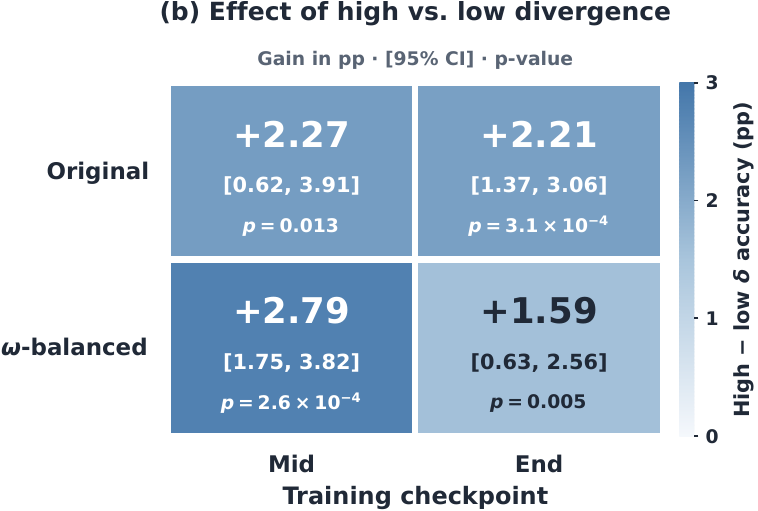}
    \caption{
    Accuracy gains of high- over low-divergence training groups.
    Each cell reports the gain in percentage points,
    its 95\% confidence interval, and $p$-value.
    }
    \label{fig:divergence_controlled_comparison}
\end{wrapfigure}

\subsection{Experimental Details and Additional Signal Validation}
\label{app:signal_validation}

\paragraph{Data and splits.}
All questions in these analyses are natural questions from DAPO-17k-Math; no synthesizer outputs are used. Among 17,177 eligible questions, we exclude a 128-question probe set and 500 questions reserved for token-level analysis from downstream training, leaving a pool of 16,549 questions. The one-step analysis uses 1,500 update questions disjoint from the probe set: the 500 token-level questions and 1,000 additional questions from the training pool. The one-step and downstream training analyses are conducted independently.

\paragraph{One-step learning value.}
For each update question, we apply a single plain-SGD OPD step with a learning rate of $3 \times 10^{-5}$ using four responses sampled from the base student. Each question is evaluated starting from the same base-student parameters. Its learning value is the resulting decrease in full-vocabulary $\mathrm{KL}(\pi_T \| \pi_S)$ on the fixed probe set, averaged over tokens of one fixed base-student response per probe question. This measurement differs from the OPD update in its questions, prefixes, and KL direction: the update uses a reverse-KL-style objective on the student's top-16 candidate tokens, whereas the probe evaluates full-vocabulary teacher-to-student KL on separate questions and fixed prefixes.

\paragraph{Test--retest reliability.}
We estimate test--retest reliability on 200 of the update questions by repeating the analysis with an independent draw of four student responses per question. For each pair of quantities being correlated, we compute a reliability-corrected correlation by dividing the observed Spearman correlation by the square root of the product of their two test--retest reliabilities. We report this correction as a supplementary analysis of sensitivity to rollout-sampling noise.

\paragraph{Score-ranked training banks.}
We partition the training pool into five approximately equal-sized buckets by mean support-aware score $\bar{s}$, yielding approximately 3,310 questions per bucket. A random bank of the same size serves as a reference. Each bank is evaluated using three training seeds.

\paragraph{Factorial designs and divergence matching.}
For the $2 \times 2$ designs, we select the bottom and top tertiles of mean divergence $\bar{\delta}$, corresponding to $\bar{\delta} \leq 0.290$ and $\bar{\delta} \geq 0.404$, respectively. Within each tertile, we split questions at the median of mean compatibility $\bar{\omega}$, producing four cells of approximately 2,758 questions each. Because $\bar{\delta}$ and $\bar{\omega}$ are negatively correlated, this split can leave differences in divergence between the compatibility cells. The $\delta$-matched variant therefore divides each divergence tertile into 20 strata by $\bar{\delta}$ and splits $\bar{\omega}$ at its median within each stratum. This procedure approximately balances mean divergence between the two compatibility cells within each divergence level: 0.528 versus 0.531 in the high-divergence tertile and 0.229 versus 0.230 in the low-divergence tertile. Each resulting cell contains approximately 2,758 questions.

\paragraph{Training budgets and checkpoints.}
All banks use a common target budget of 24.37M training tokens. Runs reach this budget within 26--36 steps, including 30--31 steps for the five score buckets and the random reference bank, and consume at most 2,304 questions. No question is used more than once within a run. Seeds differ in question order and therefore in the subset consumed before reaching the budget. We evaluate the end checkpoint at the first step reaching the full budget. For the mid checkpoint, we use the first saved checkpoint reaching half the budget, with checkpoints saved every five steps. Consequently, mid checkpoints span 12.4M--18.2M tokens across banks and provide less tightly budget-controlled comparisons than end checkpoints.

\paragraph{Statistical analysis.}
Confidence intervals for the one-step analysis are obtained by bootstrapping over the 1,500 update questions. For downstream training, benchmark scores are first averaged within each run, and the training run is the unit of inference, with three seeds per bank. Pairwise contrasts use run-level two-sample $t$-tests, and effects in the $2 \times 2$ designs are evaluated using a two-way ANOVA over runs. The Jonckheere--Terpstra test assesses the ordered trend across the five score buckets using permutations over their 15 training runs. Figure~9 reports means and standard deviations across seeds.

\paragraph{Additional controlled comparisons.}
\label{app:divergence_controls}

To examine whether the training advantage associated with higher divergence persists after accounting for compatibility, we compare high- and low-divergence question banks under both the original design and a design balanced on the compatibility factor $\omega$. Figure~\ref{fig:divergence_controlled_comparison} reports accuracy differences in percentage points at intermediate and final training checkpoints, together with 95\% confidence intervals and $p$-values.

At the final checkpoint, high-divergence banks outperform low-divergence banks by $2.21$ points in the original design (95\% CI $[1.37,3.06]$, $p=3.1\times10^{-4}$). After balancing $\omega$, the advantage remains positive at $1.59$ points (95\% CI $[0.63,2.56]$, $p=0.005$). The intermediate checkpoint shows the same direction: gains of $2.27$ points in the original design (95\% CI $[0.62,3.91]$, $p=0.013$) and $2.79$ points after balancing $\omega$ (95\% CI $[1.75,3.82]$, $p=2.6\times10^{-4}$). All four confidence intervals exclude zero, supporting a positive training advantage for high-divergence banks that persists under compatibility balancing.

\section{Theoretical Properties of the Elicited-Signal Objective}
\label{app:theory}

We establish two properties of our framework: high normalized
reward certifies positions with simultaneous discrepancy and
compatibility, and compatibility enters the implemented OPD
advantages through an explicit log-ratio term. These results
characterize the elicited supervision and its connection to
the update, without assuming convergence of synthesizer
optimization.

\paragraph{Notation.}
Fix the scoring student $\pi_S$ and teacher $\pi_T$.
For a question $q$, let $y\sim\pi_S(\cdot\mid q)$ be a
nonempty response of length $1\le L\le L_{\max}$.
At position $t$, let $p_{S,t}$ and $p_{T,t}$ denote the
next-token distributions on the student-generated prefix,
and define
\[
  A_t=\Topk(p_{S,t}),\qquad
  B_t=\Topk(p_{T,t}),\qquad
  U_t=A_t\cup B_t.
\]
For a nonempty token set $C$, write
$p_{M,t}(C)=\sum_{v\in C}p_{M,t}(v)$ and let
$\bar p_{M,t}^{C}$ denote the distribution conditioned
on $C$, extended by zero outside $C$.
We assume strictly positive next-token probabilities.
All logarithms are natural.

Following Eqs.~\eqref{eq:div}--\eqref{eq:st}, write
\[
  \delta_t
  =
  \mathrm{KL}\!\left(
    \bar p_{T,t}^{U_t}
    \,\middle\|\,
    \bar p_{S,t}^{U_t}
  \right),
  \qquad
  \omega_t=p_{T,t}(A_t),
\]
and
\[
  s_t=\min\{\delta_t/\Delta,1\}\omega_t,
  \qquad
  \Delta=\delta_{\mathrm{ref}}+\epsilon>0.
\]
Thus $s_t\in[0,1]$. With
$\bar s=L^{-1}\sum_{t=1}^{L}s_t$, the reward in
Eq.~\eqref{eq:reward} is
\[
  R=P_{\mathrm{len}}P_{\mathrm{div}}
    L^{1-\alpha}\bar s,
  \qquad
  0<\alpha<1,
\]
where both penalty factors lie in $[0,1]$.
We suppress the dependence of $R$ on the sampled response
and scoring batch, and write $[x]_+=\max\{x,0\}$.

\subsection{Compatible-disagreement coverage}
\label{app:theory-coverage}

For a threshold $h\in(0,1)$, define
\[
  \mathcal E_h
  =
  \{t:\delta_t\ge h\Delta,\ \omega_t\ge h\}.
\]
These positions simultaneously exhibit teacher--student
discrepancy and teacher probability mass within the
student's candidate set.

\begin{proposition}[Compatible-disagreement coverage]
\label{prop:coverage}
For every nonempty response,
\[
  \frac{|\mathcal E_h|}{L}
  \ge
  \left[\frac{\bar s-h}{1-h}\right]_+
  \ge
  \left[
    \frac{R/L^{1-\alpha}-h}{1-h}
  \right]_+.
\]
\end{proposition}

\begin{proof}
If $t\notin\mathcal E_h$, then either
$\delta_t<h\Delta$ or $\omega_t<h$, and hence $s_t\le h$.
Since $s_t\le1$ at every position,
\[
  \bar s
  \le
  h+(1-h)\frac{|\mathcal E_h|}{L}.
\]
Rearranging, using nonnegativity, and noting that
\[
  R/L^{1-\alpha}
  =
  P_{\mathrm{len}}P_{\mathrm{div}}\bar s
  \le\bar s
\]
gives the result.
\end{proof}

Thus, a normalized reward above $h$ certifies a nonzero
fraction of positions exhibiting both discrepancy and
compatibility. Normalization by $L^{1-\alpha}$ is
necessary because the unnormalized reward also depends
on response length.

\paragraph{Expected coverage.}
Sample a scoring batch, select a question uniformly,
sample its student response, and select a response
position uniformly. Then
\[
  \Pr(t\in\mathcal E_h)
  \ge
  \left[
    \frac{\mathbb E[R/L^{1-\alpha}]-h}{1-h}
  \right]_+.
\]
This follows by taking expectations of the pointwise
inequality in the proof above and using nonnegativity.
The expectation includes question generation, student
responses, and the batch-dependent diversity factor.

\subsection{Compatibility and OPD advantages}
\label{app:theory-compat}

We next relate compatibility to the candidate advantages
used in \S\ref{sec:opd}. Unlike the coverage certificate
computed with the frozen scoring student, the following
identity holds for any rollout policy during OPD.

\begin{proposition}[Compatibility in the OPD advantage]
\label{prop:advantage}
Fix a response position and let
$p_S=p_{S,\mathrm{old}}$ be the rollout-policy
distribution. Write
\[
  A=\Topk(p_S),\qquad
  a=p_S(A),\qquad
  \omega=p_T(A).
\]
For the candidate advantages in
Eq.~\eqref{eq:opdreward},
\[
  r(v)=\bar p_S^A(v)\log\frac{p_T(v)}{p_S(v)},
  \qquad v\in A,
\]
we have
\[
  \sum_{v\in A}r(v)
  =
  -\mathrm{KL}\!\left(
    \bar p_S^A\,\middle\|\,\bar p_T^A
  \right)
  +\log\frac{\omega}{a}.
\]
If the position is scored using these same student and
teacher distributions, then $s\ge h$ implies
$\log(\omega/a)\ge\log h$.
\end{proposition}

\begin{proof}
For each $v\in A$,
\[
  \log\frac{p_T(v)}{p_S(v)}
  =
  \log\frac{\bar p_T^A(v)}{\bar p_S^A(v)}
  +\log\frac{\omega}{a}.
\]
Multiplying by $\bar p_S^A(v)$ and summing over $A$
gives the identity. The final claim follows from
$\omega\ge s\ge h$ and $a\le1$.
\end{proof}

Compatibility therefore contributes a common additive
offset to the candidate log-ratios, while the conditional
distributions determine their relative values.
When $\omega<a$, this offset is negative; at
high-scoring positions, it is bounded below by $\log h$.
The identity characterizes the fixed advantages,
rather than the resulting parameter-gradient norm.

\paragraph{Teacher truncation.}
Compatibility also has a distributional interpretation:
\[
  \mathrm{TV}(p_T,\bar p_T^A)=1-\omega,
\]
where $\mathrm{TV}(P,Q)=\frac12\sum_v|P(v)-Q(v)|$.
Indeed, the total absolute difference inside $A$ is
$1-\omega$, as is the teacher mass outside $A$.
Consequently, $s_t\ge h$ implies
\[
  \mathrm{TV}(p_{T,t},\bar p_{T,t}^{A_t})
  \le1-h.
\]
Here $\bar p_T^A$ is an analysis device: the implemented
update uses full-vocabulary teacher log-probabilities.
Moreover, a fixed top-$K$ set need not contain most of
the student's probability mass; this must be assessed
separately through $p_S(A)$.

\paragraph{Scope.}
\label{app:theory-scope}
These results characterize the scoring signal and its
relationship to the fixed OPD advantages. They do not
lower-bound parameter gradients or downstream
improvement: the scoring discrepancy is computed on
$U_t$, whereas explicit candidate-level update terms
are restricted to $A_t$, and the scoring and update
divergences have different directions. Coverage
certificates computed with the frozen scoring student
also need not persist as the student evolves during
OPD. Compatibility concerns candidate probability mass,
rather than question correctness or membership in
either model's training distribution. The effectiveness
of the elicited supervision for student learning is
evaluated empirically in \S\ref{sec:experiments}.

\clearpage
\section{Prompt Templates}
\label{app:prompts}

We provide the synthesizer and student-solving templates below. Both mathematical
and code question synthesis use fixed instructions without persona conditioning
or few-shot example problems. In the templates, \texttt{\{q\}} denotes the question
to be solved. Line wrapping is adjusted for readability.

\begin{promptcard}{PromptNavy}{SYNTHESIZER / MATH}
\promptrole{PromptNavy}{SYSTEM}
FIRST, in your private scratch-pad, think step-by-step to design a brand-new, non-trivial problem. The problem could come from any field of mathematics, including but not limited to algebra, geometry, number theory, combinatorics, prealgebra, probability, statistics, and calculus. Aim for a difficulty such that fewer than 30\% of advanced high-school students could solve it. Avoid re-using textbook cliches or famous contest problems.

The problem must be self-contained, have a well-defined answer, and require several non-obvious reasoning steps that build on each other, not a single formula lookup.

THEN, without revealing any of your private thoughts, output ONLY the problem statement. Do not include a solution, an answer, or any commentary.\tcblower
\promptrole{PromptNavy}{USER}
Generate one new, challenging mathematical reasoning problem now.
\end{promptcard}

\begin{promptcard}{PromptNavy}{SYNTHESIZER / CODE}
\promptrole{PromptNavy}{SYSTEM}
FIRST, in your private scratch-pad, think step-by-step to design a brand-new, non-trivial competitive programming problem on your own subject. Aim for a difficulty such that a strong contestant would need several non-obvious algorithmic insights, not a single library call. Avoid re-using classic textbook tasks or well-known contest problems.

The problem must be self-contained and must state its Input format, its Output format, and the constraints on the input. The program that solves it reads from standard input and writes to standard output.

THEN, without revealing any of your private thoughts, output ONLY the problem statement. Do not include a solution, sample code, an answer, or any commentary.\tcblower
\promptrole{PromptNavy}{USER}
Generate one new, challenging competitive programming problem now.
\end{promptcard}

\clearpage
\noindent{\large\bfseries Student-solving templates}\par
\smallskip
The following templates are user messages. Mathematical reasoning uses the
stage-specific formats below; the code template is shared by reward computation,
discrepancy calibration, code OPD banks, and LiveCodeBench evaluation.

\begin{promptcard}{PromptOrange}{MATH / Reward computation and discrepancy calibration}
\promptrole{PromptOrange}{USER}
\texttt{\{q\}}

Please reason step by step, and put your final answer within \texttt{\textbackslash{}boxed\{\}}.
\end{promptcard}

\begin{promptcard}{PromptOrange}{MATH / Downstream OPD on synthetic questions}
\promptrole{PromptOrange}{USER}
Solve the following math problem step by step. The last line of your response should be of the form Answer: \texttt{\$Answer} (without quotes) where \texttt{\$Answer} is the answer to the problem.

\texttt{\{q\}}
\end{promptcard}

\begin{promptcard}{PromptOrange}{MATH / Vanilla OPD on real questions}
\promptrole{PromptOrange}{USER}
Solve the following math problem step by step. The last line of your response should be of the form Answer: \texttt{\$Answer} (without quotes) where \texttt{\$Answer} is the answer to the problem.

\texttt{\{q\}}

Remember to put your answer on its own line after "Answer:".
\end{promptcard}

\begin{promptcard}{PromptOrange}{MATH / Evaluation suffix}
{\small\itshape Append to the existing evaluation prompt in the real-question format above.}\par\smallskip
Please reason step by step, and put your final answer within \texttt{\textbackslash{}boxed\{\}}.
\end{promptcard}

\begin{promptcard}{PromptOrange}{CODE / Student solving}
\promptrole{PromptOrange}{USER}
You will be given a question (problem specification) and will generate a correct Python program that matches the specification and passes all tests. You will NOT return anything except for the program.

Question: \texttt{\{q\}}

Read the inputs from stdin solve the problem and write the answer to stdout (do not directly test on the sample inputs). Enclose your code within delimiters as follows. Ensure that when the python program runs, it reads the inputs, runs the algorithm and writes output to STDOUT.
\par\smallskip{\ttfamily\small\noindent\textasciigrave\textasciigrave\textasciigrave python\\
\# YOUR CODE HERE\\
\textasciigrave\textasciigrave\textasciigrave\par}
\end{promptcard}

\end{document}